\documentclass[11pt, a4paper, copyright]{simular}

\usepackage{url}
\usepackage{nicefrac}
\usepackage{wrapfig}
\usepackage{array}
\usepackage{multirow}
\usepackage{fontawesome5}
\usepackage{tikz}
\usetikzlibrary{matrix,calc,decorations.pathreplacing,positioning,arrows.meta}

\newcommand{\red}[1]{\textcolor{red!70!black}{#1}}

\newcommand{\blue}[1]{\textcolor{blue!70!black}{#1}}

\newtheorem{theorem}{Theorem}
\newtheorem{assumption}{Assumption}
\newtheorem{lemma}{Lemma}

\title{Stateful Reasoning via Insight Replay}

\author[1]{Bin Lei}
\author[1]{Caiwen Ding}
\author[2]{Jiachen Yang}
\author[2]{Ang Li}
\author[2]{Xin Eric Wang}

\affil[1]{University of Minnesota}
\affil[2]{Simular AI}

\correspondingauthor{Bin Lei (lei00126@umn.edu), Xin Eric Wang (eric@simular.ai)}

\usepackage[numbers, sort]{natbib}
\uselogo{}
\renewcommand{\today}{}

\begin{abstract}
Chain-of-Thought (CoT) reasoning has become a foundation for eliciting multi-step reasoning in large language models, but recent studies show that its benefits do not scale monotonically with chain length: while
longer CoT generally enables a model to tackle harder problems, on a given problem, accuracy typically increases with CoT length up to a point, after which it declines. We identify a major cause of this phenomenon: as the CoT grows, the model's attention to critical insights produced earlier in the trace gradually weakens, making those insights progressively less accessible
when they are most needed.
Therefore, we propose
\textbf{InsightReplay}, a stateful reasoning approach in which the
model periodically extracts critical insights from its reasoning trace
and replays them near the active generation frontier, keeping them
accessible as the reasoning scales. Extensive experiments on a
$\mathbf{2}\!\times\!\mathbf{3}\!\times\!\mathbf{4}$ benchmark grid,
covering model scales $\{\text{8B}, \text{30B}\}$, model families
$\{\text{Qwen3.5}, \text{DeepSeek-R1-Distill-Qwen}, \text{Gemma-4}\}$,
and reasoning benchmarks $\{\text{AIME}, \text{HMMT}, \text{GPQA Diamond},
\text{LiveCodeBench v5}\}$, show that 3-round InsightReplay yields
accuracy gains across \textbf{all 24 settings}, with an averaged
improvement of $\mathbf{+1.65}$ points over standard CoT, and a largest
single-setting gain of $\mathbf{+9.2}$ points on R1-Distill-32B's
LiveCodeBench v5 subset.
Our results suggest that the effectiveness of test-time scaling depends not only on how much a model reasons, but also on whether critical intermediate insights remain accessible throughout long reasoning trajectories.

\medskip
\centerline{%
\href{https://research.simular.ai/insight-replay/}{\raisebox{-0.22em}{\includegraphics[height=1em]{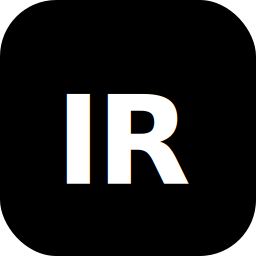}}\ Project page}%
\quad\textbar\quad
\href{https://github.com/simular-ai/insight-replay}{\faGithub\ Code}%
}
\end{abstract}

\begin{document}
\maketitle

\section{Introduction}
Chain-of-Thought (CoT) reasoning is a key mechanism for eliciting reasoning in large language models. By explicitly generating intermediate steps, CoT decomposes difficult objectives into smaller sub-problems and has become a foundation for test-time scaling: increasing reasoning-token budgets often improves performance on difficult tasks~\citep{wei2022chain,wang2023selfconsistency,guo2025deepseek,kojima2022large}. 

However, recent work shows that the benefits of additional reasoning tokens do not scale monotonically. On fixed-difficulty problems, accuracy typically follows an inverted-U relationship with CoT length (see Figure~\ref{fig:teaser}): performance first improves as the model reasons longer, reaches a peak, and then declines as the trajectory becomes excessively long~\citep{wu2025more,chen2026think,gema2025inverse}. Harder problems shift this peak rightward, requiring longer reasoning before degradation appears.

What drives this decline? Prior work has identified that not all
reasoning tokens contribute equally: a small subset of sentences disproportionately determines the final
answer, and certain attention heads consistently focus on them
\citep{bogdan2025thoughtanchorsllmreasoning}. We complement this line of
work by asking a different question: \emph{how does the model's access to such critical insights evolve as the reasoning chain grows?} 

In Section~\ref{sec:property2}, we directly measure the
pre-softmax attention from answer tokens to critical insights as
a function of token distance, and find that it decays monotonically.
As a result, later reasoning is conducted with progressively degraded access to the very intermediate conclusions that matter most.
This observation suggests a broader view of test-time scaling: increasing reasoning depth alone is insufficient unless important reasoning state can also remain accessible over time.

This echoes a classical challenge in sequence modeling: preserving long-range information across time. LSTM \citep{hochreiter1997long} addressed this by introducing an explicit memory cell that persists and is continually updated as the sequence unfolds. A natural question therefore arises for long reasoning chains: \textit{can we equip the growing reasoning chain with an analogous mechanism that keeps critical insights accessible over time?}

\begin{wrapfigure}{r}{0.5\linewidth}
  \vspace{-0.6em}
  \centering
  \includegraphics[width=\linewidth]{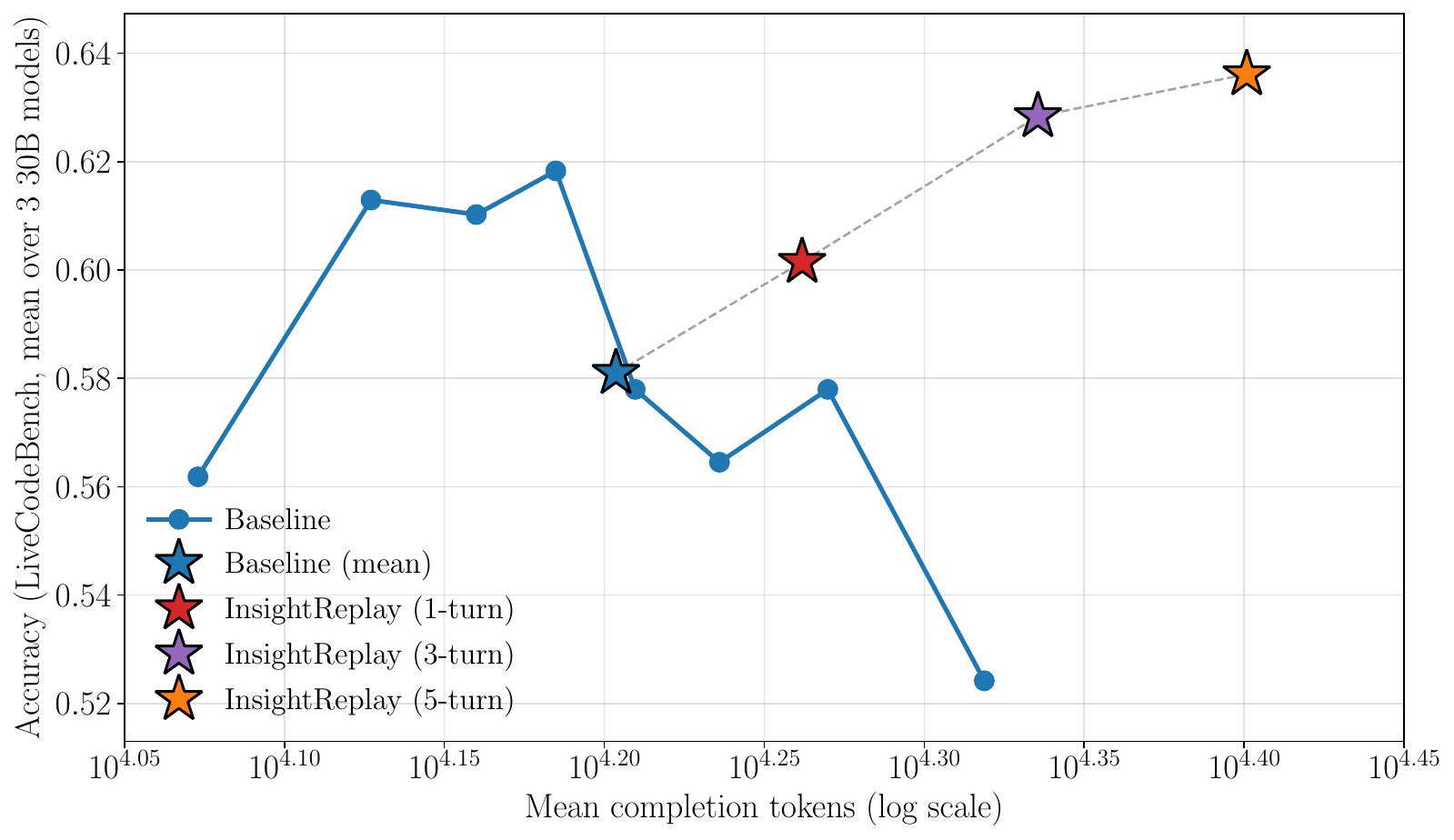}
  \caption{
  Accuracy vs.\ mean tokens for Baseline (standard 
  CoT) and InsightReplay ($1$, $3$, $5$ replay rounds), averaged over three 
  30B-tier models on the LiveCodeBench v5 subset. Each star reports 
  the overall mean accuracy.
  % See Section~\ref{sec:inverted-u-validation} for the full setup.
    % \eric{one concern for this figure: reviewers might question that under the same token budget (before the peak), InsightReplay performs worse than CoT.}
  }
  \label{fig:teaser}
  \vspace{-0.6em}
\end{wrapfigure}

% To this end, we introduce \emph{stateful reasoning} for CoT, 
% where the reasoning state at any point is the cumulative set of 
% \emph{insights} the model has generated so far, compressed 
% abstractions of prior reasoning. We instantiate this with 
% \textbf{InsightReplay}, which periodically generates new 
% insights and replays them near the active generation frontier. 
% By keeping these insights close to the decoding position where 
% attention to them remains strong, InsightReplay mitigates the 
% attention decay that critical intermediate content suffers in 
% long trajectories, raising the achievable peak accuracy and 
% expanding the regime in which extended reasoning translates 
% into higher accuracy\eric{this sentence is too long. break it into several sentences.}.

To answer this question, we introduce \emph{stateful reasoning} for CoT, where 
the reasoning state at any point is the cumulative set of 
\emph{insights} the model has generated so far, compressed 
abstractions of prior reasoning. We instantiate this with 
\textbf{InsightReplay}, which periodically generates new insights 
and replays them near the active generation frontier. Keeping 
insights close to the decoding position preserves the model's 
attention to them. This mitigates the attention decay that critical 
intermediate content suffers in long trajectories, raising the 
achievable peak accuracy and expanding the regime in which extended 
reasoning translates into higher accuracy.

Figure~\ref{fig:teaser} previews this property. The Baseline curve exhibits the characteristic inverted-U pattern: accuracy first rises with CoT length, peaks at around $15$K tokens, and then declines as the chain grows further.
InsightReplay behaves differently. It operates precisely in the long CoT regime where additional reasoning already degrades the Baseline, yet each additional replay round ($1$, $3$, $5$) achieves progressively higher mean accuracy. 
By $3$ replay rounds, InsightReplay's mean accuracy already exceeds the Baseline's peak; by $5$ rounds, performance continues to improve further. In effect, InsightReplay turns the Baseline's degradation regime into a continued-growth regime, 
extending the range of token budgets over which additional reasoning remains productive.

% The main contributions of this paper are as follows:
% \begin{itemize}
%     \item We empirically study critical insights in CoT and identify two 
%     key properties: \textbf{(I)} critical insights and the surrounding 
%     reasoning trace 
%     are complementary; \textbf{(II)} attention to critical insights 
%     decays as CoT extends, making them progressively less accessible 
%     to the model.
%     % \item Building on these two properties, we develop a theoretical model\eric{this is a big claim. Most deep learning methods are experiment-driven and not supported by theory. I would avoid the overclaiming risk.} showing that InsightReplay shifts the optimal reasoning length to the right and lifts peak accuracy.
%     \item Building on these two properties, we propose 
%     \textbf{InsightReplay}, a stateful reasoning approach that 
%     periodically extracts critical insights and replays them near the 
%     active generation frontier. We provide a theoretical analysis 
%     showing that InsightReplay can raise the peak accuracy.
%     \item We validate InsightReplay in a purely inference-time setting and show that replaying critical insights yields substantial gains without any additional training.
%     \item We further reinforce this reasoning pattern via post-training, demonstrating the advantages of InsightReplay over vanilla CoT under the same training settings.
% \end{itemize}

The main contributions of this paper are as follows:

\begin{itemize}

    \item We empirically study critical insights in CoT and identify two key properties: \textbf{(I)} critical insights and the surrounding reasoning trace are complementary; \textbf{(II)} attention to critical insights decays as CoT grows, making them progressively less accessible to the model.

    \item Building on these observations, we propose \textbf{InsightReplay}, a stateful reasoning approach that periodically extracts critical insights and replays them near the active generation frontier. We further provide a theoretical analysis showing that InsightReplay can shift the optimal reasoning length rightward and raise the achievable peak accuracy.

    \item We validate InsightReplay in a purely inference-time setting and show that replaying critical insights yields consistent improvements across a broad range of models and reasoning benchmarks, without any additional training.

    \item We further reinforce this reasoning pattern through post-training, demonstrating that InsightReplay improves both training stability and reasoning performance over vanilla CoT under same training settings.

\end{itemize}

\section{Properties of Critical Insights in Long CoT}

\subsection{Critical Insights and the Reasoning Trace are Complementary}
\label{sec:property1}

\paragraph{Related Work.}
Recent work provides evidence that information in CoT is not uniformly distributed: \emph{compression methods} show that many reasoning tokens can be pruned with minimal accuracy loss \citep{xia2025tokenskip,kang2025c3ot,yuan2025not,li2026making}; \emph{summarize-and-restart methods} first generate a reasoning segment, then summarizes it into a compact form, discards the original segment, and resumes thinking from the summary \citep{yan2025inftythink,yang2025pencil,madaan2025rethinking}; and \emph{importance-attribution methods} identify a small number of sentences or tokens that disproportionately influence the final answer \citep{bogdan2025thoughtanchorsllmreasoning,lin2024critical}. What remains underexplored is the relationship between the original insights produced within the CoT and the surrounding reasoning trace that produced them. We seek to verify and quantify whether the two are substitutable or complementary.

\paragraph{Experimental Setting.}
Using \texttt{Qwen3-8B} as the base model, we generate thinking traces for $60$ AIME competition problems~\citep{maa_aime}, with a maximum generation length of $30{,}000$ tokens. We then use the same model to extract $5$--$7$ key insights from each trace.
We construct seven conditions with identical chat templates and answer formats, differing only in the content inside the thinking tag:
\textbf{(A)} original full thinking;
\textbf{(B)} repeated question + insights;
\textbf{(C)} insights only;
\textbf{(D)} full thinking + repeated question + insights;
\textbf{(E)} full thinking + insights;
\textbf{(F)} random tokens + repeated question + insights;
\textbf{(G)} random tokens + insights.
For each condition, we compute the answer probability $P(\mathrm{ans}) := P_{\theta}(a_{\mathrm{gt}} \mid x) = \prod_{t} P_{\theta}(a_{\mathrm{gt}}^{(t)} \mid x, a_{\mathrm{gt}}^{(<t)})$ via a single forward pass on \texttt{Qwen3-8B}, where $x$ is the prompt prefix up to and including \texttt{\textbackslash boxed\{} and $a_{\mathrm{gt}}$ are the ground-truth answer tokens; full details are in Appendix~\ref{app:property1}.

\begin{table}[ht]
\centering
\small
\caption{Effect of thinking content on answer probability $P(\mathrm{ans})$ on $60$ AIME problems with \texttt{Qwen3-8B}. $\checkmark$ = present, -- = absent, $\varnothing$ = random tokens length-matched to CoT reasoning trace.}
\label{tab:cot-critical-insights}
\begin{tabular}{c ccc ccc}
\toprule
\multirow{2}{*}{\textbf{Conditions}} & \multicolumn{3}{c}{\textbf{Content inside \texttt{<think>}}} & \multirow{2}{*}{\textbf{Tokens}} & \multirow{2}{*}{\textbf{$P(\mathrm{ans})$}} & \multirow{2}{*}{\textbf{$P/\text{Tok}$ (\textperthousand)}} \\
\cmidrule(lr){2-4}
 & CoT reasoning trace & Repeated Q & Insights & & & \\
\midrule
\rowcolor{gray!15}
A & \checkmark & --         & --         & $16{,}731$ & $0.512$            & $0.031$ \\
\midrule
B & --         & \checkmark & \checkmark & $378$      & $0.273$            & $0.722$ \\
C & --         & --         & \checkmark & $236$      & $0.387$            & $\mathbf{1.640}$ \\
D & \checkmark & \checkmark & \checkmark & $17{,}109$ & $\mathbf{0.557}$   & $0.033$ \\
E & \checkmark & --         & \checkmark & $16{,}967$ & $0.545$            & $0.032$ \\
F & $\varnothing$ & \checkmark & \checkmark & $17{,}109$ & $0.131$            & $0.008$ \\
G & $\varnothing$ & --         & \checkmark & $16{,}967$ & $0.127$            & $0.007$ \\
\bottomrule
\end{tabular}
\end{table}

\paragraph{Results \& Conclusion.}
Table~\ref{tab:cot-critical-insights} reveals four observations:
\textbf{(I) Insights carry concentrated signal.} Condition C 
retains $75.6\%$ of the answer probability with only $1.4\%$ 
of the tokens, indicating that the core reasoning content 
distills into a small set of critical insights.
\textbf{(II) Insights provide additive value on top of the reasoning trace.} Replaying insights on top of the original reasoning trace further improves accuracy from $0.512$ to $0.545$--$0.557$, demonstrating that even when the reasoning trace is already present, re-exposing the model to its own insights yields measurable gains.
\textbf{(III) Bad reasoning trace actively harms inference.} As an 
extreme case, F and G replace the entire CoT reasoning trace with length-matched random tokens, dropping accuracy to $0.127$--$0.131$, which is below the insight-only conditions B 
($0.273$) and C ($0.387$).
\textbf{(IV) Repeating the question helps only when paired with 
a reasoning trace.} The repeated question helps with a real 
reasoning trace (D vs.\ E), but hurts substantially 
without any reasoning trace (B vs.\ C).

\subsection{Critical Insights Become Less Accessible as Reasoning Chains Grow}
\label{sec:property2}

\paragraph{Related Work.}
While prior work has 
identified important reasoning steps and the heads that attend 
to them \citep{bogdan2025thoughtanchorsllmreasoning}, and while 
positional encodings such as RoPE 
\citep{su2023roformerenhancedtransformerrotary} are known to 
bias attention toward nearby tokens, empirical attention 
patterns are ultimately shaped by learned query-key interactions 
rather than positional geometry alone 
\citep{barbero2025roundroundgomakes}. Whether attention to 
critical insights actually decays with distance therefore 
requires direct measurement. We take critical insights as the 
unit of analysis and measure how attention from answer tokens 
to these insights varies with token distance.

\paragraph{Experimental Setting.}
On $60$ AIME problems with \texttt{Qwen3-8B} and 
\texttt{Gemma-4-31B-it}, the model first generates a CoT 
trace, after which \texttt{Qwen3.5-397B-A17B} identifies $5$--$7$ 
critical insights, each localized to a token span $c$ ending at 
position $p_e$. We use a stronger extractor here because attention measurement requires 
precise token boundaries that smaller models do not produce 
reliably.
To isolate the effect of distance, we insert semantically neutral 
filler tokens at sentence boundaries between the insight 
and the answer. Each filler unit is \texttt{"\textbackslash n~"} 
(newline + space), chosen for its minimal semantic impact. Let 
$p_a$ denote the answer position, so the baseline distance is 
$d = p_a - p_e$. We sweep an insertion ratio 
$\rho \in \{0, 0.1, 0.2, 0.3, 0.4\}$, inserting $N = \lfloor \rho 
d \rfloor$ filler units total, distributed uniformly at random 
across the sentence boundaries between insight and answer. The 
resulting total distance is $(1+\rho) d$ regardless of placement. 
We use one sample at $\rho = 0$ and $10$ random allocations at 
each non-zero $\rho$, yielding $41$ variants per insight. Since the insight and all preceding tokens are untouched, any 
change in attention is primarily attributable to distance.
For each variant, we compute the per-layer
\textbf{pre-softmax} attention score from the answer token to the insight span:
\[
  s^{(\ell)}(c, a) \;=\; \frac{1}{|H|}\sum_{h}\;\frac{1}{|c|}\sum_{t \in c}\; \frac{\mathbf{q}_{p_a}^{(\ell, h)} \cdot \mathbf{k}_{t}^{(\ell, h)}}{\sqrt{d_\mathrm{head}}},
\] 
where $c$ denotes the set of token positions covered by the insight span (with $|c|$ its size and $t$ ranging over its elements), $\ell$ is the layer index, $h$ indexes attention heads, $\mathbf{q}$ and $\mathbf{k}$ are post-RoPE query and key vectors, $|H|$ is the number of heads, and $d_\mathrm{head}$ is the per-head dimension. We use pre-softmax scores rather than post-softmax weights because softmax normalization over the full preceding context mechanically squashes weights as context length grows, confounding the effect of distance with that of sequence length. Full implementation details are in Appendix~\ref{app:property2}.

\begin{figure}[ht]
\centering
\includegraphics[width=\textwidth]{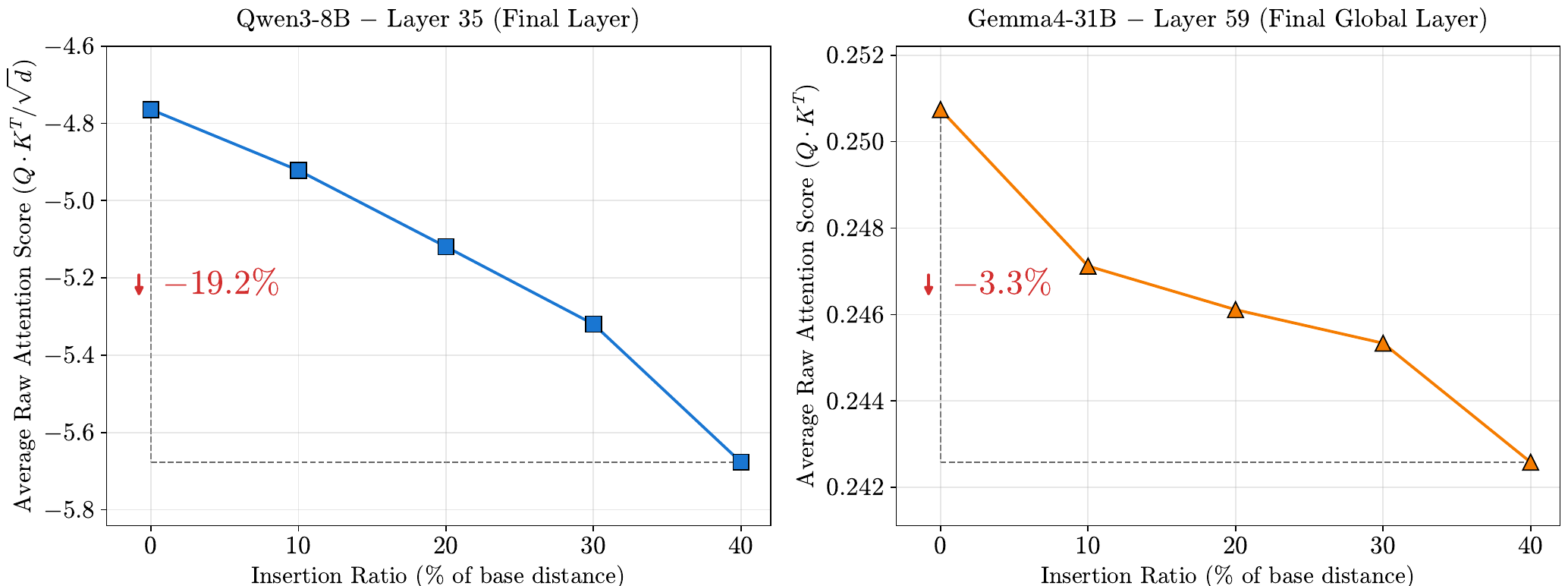}
\caption{Average pre-softmax attention from the answer token to 
insights vs.\ insertion ratio $\rho$ at each model's final 
attention layer, under native scaling 
($\mathbf{q}\cdot\mathbf{k}/\sqrt{d_\mathrm{head}}$ for 
\texttt{Qwen3-8B}; $\mathbf{q}\cdot\mathbf{k}$ for \texttt{Gemma-4-31B-it}). Annotations show relative decline from $\rho{=}0-0.4$ (paired bootstrap, $p<0.001$).}
\label{fig:final-layer-comparison}
\end{figure}

\paragraph{Results \& Conclusion}
Figure~\ref{fig:final-layer-comparison} reports the relative 
decline in attention to critical insights from $\rho{=}0$ to 
$\rho{=}0.4$: $19.2\%$ for \texttt{Qwen3-8B} (final layer) and $3.3\%$ 
for \texttt{Gemma-4-31B-it} (final global-attention layer). 
We draw two conclusions:
\textbf{(I)} Within each model, attention to critical insights 
decays monotonically as the insight is pushed away from the 
generation frontier, and the decline is highly significant 
($p<0.001$).
\textbf{(II)} The decay is not specific to a single architecture 
or RoPE configuration: it persists in Gemma-4-31B-it, which uses 
partial-rotary RoPE on only the first quarter of head dimensions 
and a different attention scaling, suggesting the phenomenon 
reflects a general property of trained attention rather than a 
feature of any specific positional encoding.

\section{InsightReplay}
\subsection{Method Overview}
Figure~\ref{fig:insightreplay-method} illustrates the core idea 
of InsightReplay. Reasoning unfolds as an interleaving of two 
types of content: \emph{reasoning chunks} ($R_t$), in which the 
model performs step-by-step thinking, and \emph{insights} 
($I_t$), short distilled conclusions the model itself produces 
to summarize prior reasoning. As the trajectory progresses, the 
sequence of insights $\{I_t\}$ acts as a continuously updated, 
compressed abstraction of what has been concluded so far, kept 
near the active generation frontier so it remains accessible to 
subsequent reasoning. One specific example is shown in Appendix~\ref{app:case-analysis}.

\begin{figure}[ht]
\centering
\includegraphics[width=\textwidth]{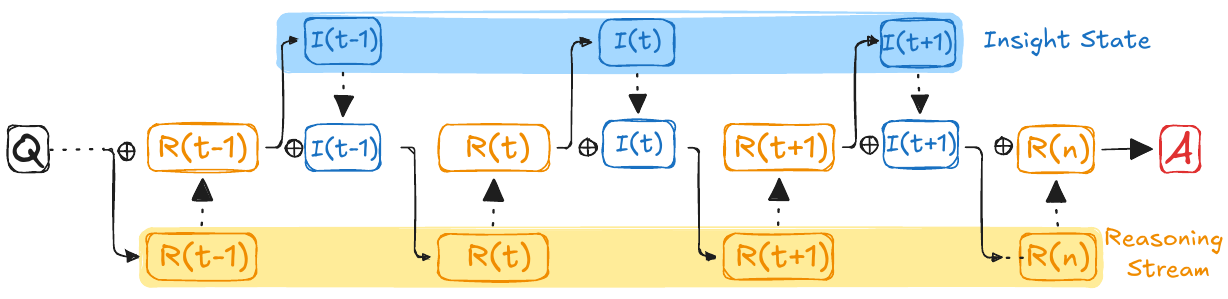}
\caption{Illustration of InsightReplay. \includegraphics[height=0.8em]{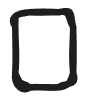}: input question, \includegraphics[height=0.8em]{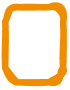}: reasoning chunk, \includegraphics[height=0.8em]{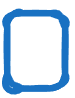}: insight, \includegraphics[height=0.8em]{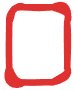}: final answer, \includegraphics[height=0.8em]{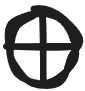}: appending a newly generated chunk to the conversation history, \includegraphics[height=0.8em]{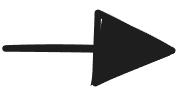}: forward inference, and \includegraphics[height=0.8em]{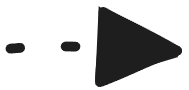}: grouping chunks into a logical state. The dashed arrows group chunks into two logical roles: insights collectively form the \textcolor{cyan!50!blue}{\textbf{Insight State}}, and reasoning chunks form the \textcolor{orange!85!black}{\textbf{Reasoning Stream}}. These groupings reflect logical roles only, and all chunks reside in the same conversation history.}
\label{fig:insightreplay-method}
\end{figure}
\paragraph{Method Description.}
The process begins from the user question $Q$. We initialize 
$I_0 \!=\! Q$, treating the question itself as a ``goal 
insight'' that anchors all subsequent reasoning, consistent 
with Section~\ref{sec:property1}. The model then enters a 
periodic loop. At each round $t$, conditioned on the full 
history $(Q, R_1, I_1, \ldots, R_{t-1}, I_{t-1})$, the model:
(i) generates a reasoning chunk $R_t$, and
(ii) generates a new insight $I_t$ via a summarization prompt that asks the model to consolidate the conclusions reached so far, with the thinking content $R_1, \ldots, R_t$ retained in context.
Both $R_t$ and $I_t$ are appended to the history in order. 
Crucially, because $I_t$ is generated in the presence of all 
prior insights $I_0, \ldots, I_{t-1}$, each new insight can 
reflect on and supersede earlier ones---incorporating 
new conclusions, correcting prior errors, or refining 
imprecise statements. The cumulative trace of insights thus 
forms an evolving abstraction of the reasoning so far, rather 
than a flat concatenation of independent summaries. After $N$ 
rounds (a hyperparameter; we ablate $N \!\in\! \{1, 3\}$ in 
Section~\ref{sec:experiments}), the model produces the final 
answer $A$ from the complete history. By construction, the 
most recent $I_{t-1}$ always sits immediately before $R_t$, 
keeping the latest distilled conclusions close to the decoding 
position of subsequent reasoning.

\subsection{Theoretical Analysis}
\label{sec:theory}
In this section, we formalize the observations from Sections~\ref{sec:property1} and~\ref{sec:property2}, and prove that InsightReplay provably shifts the optimal CoT length rightward and raises the peak accuracy.  

\subsubsection{Modeling}
%==============================================================================

Following \citet{wu2025more}, for a task $q$ with difficulty $T$ and model $\theta$, the probability of an $N$-step CoT producing the correct final answer is:
\begin{equation}
    P(a_{\mathrm{final}} | q, \theta, N)
    = \prod_{i=1}^{N}
      \underbrace{P(q_i | H_{i-1}, q, \theta, N)}_{\text{sub-question}}
      \;\underbrace{P(a_i | q_i, H_{i-1}, q, \theta, N)}_{\text{sub-answer}},
    \label{eq:cot_def}
\end{equation}
where $H_k := [t_1, a_1, \dots, t_k, a_k]$ denotes the reasoning history up to step~$k$.

They parameterize the sub-question accuracy as $(1 - \sigma(T))$, where $\sigma(T)$ denotes the probability that the model fails to generate the correct next sub-question for a task of total difficulty $T$, and the sub-answer accuracy as $(1 - E(N, M, T))$ with a linear error model $E(N,M,T) = T/(NM)$, where $M$ represents the model's reasoning capability. Thus, the accuracy of $N$-step reasoning is:
\begin{equation}
    A_0(N) = \alpha\,\bigl[(1-\sigma(T))(1-T/(NM))\bigr]^N,
    \label{eq:A0}
\end{equation}
where $\alpha$ is a constant independent of $N$.
Crucially, their parameterization \emph{drops the dependence on the history} $H_{i-1}$, implicitly assuming each step's accuracy is independent of its position in the chain.

However, our Section~\ref{sec:property2} shows: as $i$ increases, the model's ability to access critical insights from earlier steps diminishes due to attention decay.
We restore this missing dependence by introducing an \emph{insight accessibility function} $\Phi(i)$ into the sub-question term.
Specifically, we model the sub-question accuracy at step $i$ as $(1 - \sigma(T)) \cdot \Phi(i)$, where $\Phi(i)$ captures how effectively the model can access prior insights when formulating the $i$-th sub-question. We apply $\Phi(i)$ to the sub-question term because generating the correct next sub-question requires synthesizing prior insights, whereas answering a sub-question, once correctly posed, depends primarily on the sub-question itself.

\begin{assumption}
\label{assump:decay}
The insight accessibility function $\Phi: [0, \infty) \to (0, 1]$ satisfies:
\begin{enumerate}
    \item $\Phi(0) = 1$ \hfill \textup{(zero distance $\Rightarrow$ full accessibility);}
    \item $\Phi$ is strictly monotonically decreasing (Section~\ref{sec:property2}) \hfill \textup{(farther $\Rightarrow$ less accessible);}
    \item $\displaystyle\lim_{i \to \infty} \Phi(i) = 0$ \hfill \textup{(infinite distance $\Rightarrow$ inaccessible).}
\end{enumerate}
\end{assumption}

Under this modeling, we define two accuracy functions:

% \paragraph{Real-world CoT (with insight decay).}
% \begin{equation}
%     A_\Phi(N)
%     = \prod_{i=1}^{N} (1-\sigma)(1-E)\,\Phi(i)
%     = [p(N)]^N \cdot \prod_{i=1}^{N} \Phi(i),
%     \label{eq:Aphi}
% \end{equation}
% where $p(N) := (1 - \sigma(T))(1 - E(N, M, T))$ denotes the per-step accuracy without decay.

% \paragraph{InsightReplay.}
% InsightReplay relocates critical insights to a fixed distance $d_0 \geq 0$ from the current generation point at every step, so each step's accessibility becomes $\Phi(d_0)$ instead of $\Phi(i)$:
% \begin{equation}
%     A_{\mathrm{IR}}(N)
%     = [p(N)]^N \cdot [\Phi(d_0)]^N
%     = [p(N) \cdot \Phi(d_0)]^N.
%     \label{eq:Air}
% \end{equation}

\paragraph{Real-world CoT vs.\ InsightReplay.}
We define two accuracy functions, where 
$p(N) := (1 - \sigma(T))(1 - E(N, M, T))$ denotes the per-step 
accuracy without decay. The real-world CoT accuracy $A_\Phi$ accounts 
for insight decay via $\Phi(i)$, while InsightReplay relocates 
critical insights to a fixed distance $d_0\geq 0$ from the generation 
frontier at every step, replacing $\Phi(i)$ with the constant 
$\Phi(d_0)$:
\begin{equation}
    A_\Phi(N) = [p(N)]^N \prod_{i=1}^{N} \Phi(i),
    \qquad
    A_{\mathrm{IR}}(N) = [p(N) \cdot \Phi(d_0)]^N.
    \label{eq:A_both}
\end{equation}

\subsubsection{Main Results}

We first state two structural lemmas that underlie our main results. 
Let $g(N) := \frac{d}{dN}[N\ln p(N)]$.

\begin{lemma}[Strict concavity]
\label{lem:concave}
$\ln A_0(N)$ is strictly concave for $N>T/M$, and consequently $g$ is 
strictly decreasing on this domain.
\end{lemma}

\begin{lemma}[Uniqueness of optima]
\label{lem:unique}
$A_\Phi$ and $A_{\mathrm{IR}}$ each admit a unique maximizer on 
$(0,\infty)$, characterized by 
\begin{equation}
    g(N_\Phi^*) = -\ln\Phi(N_\Phi^*), \qquad 
    g(N_{\mathrm{IR}}^*) = -\ln\Phi(d_0).
\end{equation}
\end{lemma}

Building on these lemmas, our two main results show that 
InsightReplay both shifts the optimal CoT length and raises the 
achievable peak accuracy.

\begin{theorem}[InsightReplay shifts the optimal length rightward]
\label{thm:ir_shift}
$N_{\mathrm{IR}}^* > N_\Phi^*$.
\end{theorem}

\begin{theorem}[InsightReplay raises peak accuracy]
\label{thm:ir_peak}
$A_{\mathrm{IR}}(N_{\mathrm{IR}}^*) > A_\Phi(N_\Phi^*)$.
\end{theorem}

\paragraph{Proof intuition.}
Theorem~\ref{thm:ir_shift} follows because InsightReplay places 
insights adjacent to the generation frontier ($d_0 \ll N_\Phi^*$), so 
$\Phi(d_0) > \Phi(N_\Phi^*)$. Combined with the optimality conditions 
in Eq.~\eqref{eq:opt_cond} and the strict monotonicity of $g$ 
(Lemma~\ref{lem:concave}), this forces $N_{\mathrm{IR}}^* > N_\Phi^*$. 
Theorem~\ref{thm:ir_peak} follows because $\Phi(d_0) > \Phi(i)$ for 
all $i\geq 1$ implies $A_{\mathrm{IR}}(N) > A_\Phi(N)$ pointwise, and 
maximizing both sides preserves the inequality.
Full proofs are deferred to Appendix~\ref{app:proofs}.

\vspace{-0.5em}
\section{Pure Inference Experiments}
\label{sec:experiments}
\vspace{-0.5em}
We evaluate InsightReplay across a $2 \times 3 \times 4$ experimental grid: 
two model scales, three model families, and four reasoning benchmarks. Full 
implementation details are provided in Appendix~\ref{app:experiment-details}.
\vspace{-0.5em}
\subsection{Experimental Setup}
\paragraph{Models.}
We evaluate $6$ open-weight reasoning models, organized into a $2 \times 3$ grid of scales and families:
\textbf{(1)} \texttt{Qwen3.5-9B} and \texttt{Qwen3.5-35B-A3B}~\citep{qwen35}, where the latter is a mixture-of-experts model with $3$B active parameters;
\textbf{(2)} \texttt{Gemma-4-E4B} and \texttt{Gemma-4-31B-it}~\citep{gemma4}, dense models that interleave local sliding-window attention ($1024$ tokens) with global full-context attention and apply partial-rotary RoPE on the first quarter of head dimensions in global layers;
\textbf{(3)} \texttt{R1-Distill-Qwen-7B} and \texttt{R1-Distill-Qwen-32B}~\citep{guo2025deepseek}, distilled from DeepSeek-R1 onto Qwen2.5 backbones.
\paragraph{Benchmarks.}
We evaluate on four reasoning 
benchmarks: \textbf{AIME 2025/2026}~\citep{maa_aime} (60 math 
problems with integer answers), \textbf{HMMT}~\citep{balunovic2025matharena} 
(60 Harvard--MIT competition math problems), \textbf{GPQA 
Diamond}~\citep{gpqa} (198 graduate-level science MCQs), and 
\textbf{LiveCodeBench v5}~\citep{livecodebench} (62 code 
problems uniformly sampled from the full benchmark; see 
Appendix~\ref{app:experiment-details} for the sampling procedure).
\paragraph{Variants.}
We compare four inference variants: 
\textbf{(1) Base}: standard thinking-mode CoT;
\textbf{(2) VO} (\emph{Verify-Only}): the model is prompted to 
continue thinking after its initial trace via a 
budget-forcing cue (\texttt{"Wait, let me double-check..."}), but performs no insight 
extraction or replay;
\textbf{(3) IR1} / \textbf{(4) IR3}: InsightReplay with 1 or 3 
rounds of extract-then-replay.
The contrast between IR$k$ and VO isolates the contribution of 
the InsightReplay mechanism beyond extended thinking alone.
\vspace{-0.5em}
\paragraph{Sampling and metrics.}
For each problem, we draw $16$ completions with temperature $\tau = 1$ and top-$p = 0.95$. The model generates until it produces an EOS token or reaches its native context length. We report accuracy as the primary metric, with average completion tokens as a cost measure.

% ──────────────────────────────────────────────────
%  Main Results Table
% ──────────────────────────────────────────────────
\vspace{-0.5em}
\begin{table*}[ht]
  \centering
  \caption{%
    Main inference results on four reasoning benchmarks.
    \textbf{Base}: standard thinking-mode CoT inference.
    \textbf{VO} (Verify-Only): a control for test-time scaling, in which the model is instructed to keep thinking after its initial trace but performs no insight extraction or replay.
    \textbf{IR1}/\textbf{IR3}: InsightReplay with 1 or 3 rounds of extract-then-replay.
    Accuracy gains over Base are highlighted in \textcolor{blue}{blue} when positive and \textcolor{red}{red} when negative; token columns report average completion tokens per sample.
  }
  \label{tab:main-results}
  \vspace{4pt}
  \scriptsize
  \setlength{\tabcolsep}{2.0pt}

  \resizebox{\textwidth}{!}{%
  \begin{tabular}{c c c c c c c c c c c c c c c c}
  \toprule
  \multirow{2}{*}{\textbf{Model}} & \multirow{2}{*}{\textbf{Data}} & \multicolumn{7}{c}{\textbf{Accuracy}} & \multicolumn{7}{c}{\textbf{Avg. completion tokens / sample}} \\
  \cmidrule(lr){3-9} \cmidrule(lr){10-16}
  & & \textbf{Base} & \textbf{VO} & \textbf{IR1} & \textbf{IR3} & \textbf{$\Delta$VO} & \textbf{$\Delta$IR1} & \textbf{$\Delta$IR3} & \textbf{Base} & \textbf{VO} & \textbf{IR1} & \textbf{IR3} & \textbf{VO/B} & \textbf{IR1/B} & \textbf{IR3/B} \\
  \midrule
  \multirow{4}{*}{$\mathrm{Qwen3.5\text{-}9B}$}
    & $\mathrm{AIME}$ & $88.30$ & $88.30$ & $88.70$ & $89.30$ & \blue{$+0.00$} & \blue{$+0.40$} & \blue{$\mathbf{+1.00}$} & $37102$ & $40835$ & $41952$ & $49404$ & $1.10\times$ & $1.13\times$ & $1.33\times$ \\
    & $\mathrm{GPQA}$ & $80.90$ & $81.20$ & $81.20$ & $81.10$ & \blue{$\mathbf{+0.30}$} & \blue{$\mathbf{+0.30}$} & \blue{$+0.20$} & $15587$ & $18104$ & $18696$ & $22584$ & $1.16\times$ & $1.20\times$ & $1.45\times$ \\
    & $\mathrm{LCB}$  & $52.50$ & $56.60$ & $58.20$ & $61.20$ & \blue{$+4.10$} & \blue{$+5.70$} & \blue{$\mathbf{+8.70}$} & $48097$ & $52275$ & $52526$ & $58582$ & $1.09\times$ & $1.09\times$ & $1.22\times$ \\
    & $\mathrm{HMMT}$ & $83.50$ & $83.90$ & $84.20$ & $84.50$ & \blue{$+0.40$} & \blue{$+0.70$} & \blue{$\mathbf{+1.00}$} & $45477$ & $49796$ & $51088$ & $58490$ & $1.09\times$ & $1.12\times$ & $1.29\times$ \\
  \midrule
  \multirow{4}{*}{$\mathrm{Gemma4\text{-}E4B}$}
    & $\mathrm{AIME}$ & $45.40$ & $46.00$ & $46.00$ & $46.10$ & \blue{$+0.60$} & \blue{$+0.60$} & \blue{$\mathbf{+0.70}$} & $5874$  & $7453$  & $7414$  & $10199$ & $1.27\times$ & $1.26\times$ & $1.74\times$ \\
    & $\mathrm{GPQA}$ & $56.40$ & $56.60$ & $56.70$ & $57.00$ & \blue{$+0.20$} & \blue{$+0.30$} & \blue{$\mathbf{+0.60}$} & $3342$  & $4286$  & $4266$  & $6023$  & $1.28\times$ & $1.28\times$ & $1.80\times$ \\
    & $\mathrm{LCB}$  & $56.50$ & $56.60$ & $57.30$ & $57.90$ & \blue{$+0.10$} & \blue{$+0.80$} & \blue{$\mathbf{+1.40}$} & $6705$  & $8191$  & $8052$  & $10357$ & $1.22\times$ & $1.20\times$ & $1.54\times$ \\
    & $\mathrm{HMMT}$ & $36.40$ & $37.70$ & $36.70$ & $36.60$ & \blue{$\mathbf{+1.30}$} & \blue{$+0.30$} & \blue{$+0.20$} & $7481$  & $9203$  & $9064$  & $11906$ & $1.23\times$ & $1.21\times$ & $1.59\times$ \\
  \midrule
  \multirow{4}{*}{$\mathrm{R1\text{-}Distill\text{-}7B}$}
    & $\mathrm{AIME}$ & $43.80$ & $44.00$ & $44.50$ & $44.20$ & \blue{$+0.20$} & \blue{$\mathbf{+0.70}$} & \blue{$+0.40$} & $12377$ & $13432$ & $14144$ & $17879$ & $1.09\times$ & $1.14\times$ & $1.44\times$ \\
    & $\mathrm{GPQA}$ & $50.00$ & $49.00$ & $50.20$ & $50.40$ & \red{$-1.00$}  & \blue{$+0.20$} & \blue{$\mathbf{+0.40}$} & $6998$  & $8059$  & $7868$  & $10246$ & $1.15\times$ & $1.12\times$ & $1.46\times$ \\
    & $\mathrm{LCB}$  & $7.30$  & $8.90$  & $8.50$  & $8.50$  & \blue{$\mathbf{+1.60}$} & \blue{$+1.20$} & \blue{$+1.20$} & $11631$ & $12856$ & $12438$ & $14217$ & $1.11\times$ & $1.07\times$ & $1.22\times$ \\
    & $\mathrm{HMMT}$ & $28.30$ & $29.10$ & $29.00$ & $29.00$ & \blue{$\mathbf{+0.80}$} & \blue{$+0.70$} & \blue{$+0.70$} & $14417$ & $15449$ & $15818$ & $19853$ & $1.07\times$ & $1.10\times$ & $1.38\times$ \\
  \midrule
  \multirow{4}{*}{$\mathrm{Qwen3.5\text{-}35B}$}
    & $\mathrm{AIME}$ & $91.50$ & $91.70$ & $91.80$ & $92.00$ & \blue{$+0.20$} & \blue{$+0.30$} & \blue{$\mathbf{+0.50}$} & $30110$ & $33584$ & $34470$ & $40966$ & $1.12\times$ & $1.14\times$ & $1.36\times$ \\
    & $\mathrm{GPQA}$ & $84.00$ & $84.10$ & $84.10$ & $84.40$ & \blue{$+0.10$} & \blue{$+0.10$} & \blue{$\mathbf{+0.40}$} & $19767$ & $23228$ & $23248$ & $28101$ & $1.18\times$ & $1.18\times$ & $1.42\times$ \\
    & $\mathrm{LCB}$  & $67.60$ & $68.50$ & $70.60$ & $71.70$ & \blue{$+0.90$} & \blue{$+3.00$} & \blue{$\mathbf{+4.10}$} & $31538$ & $35301$ & $35886$ & $41641$ & $1.12\times$ & $1.14\times$ & $1.32\times$ \\
    & $\mathrm{HMMT}$ & $88.50$ & $88.50$ & $88.40$ & $88.60$ & \blue{$+0.00$} & \red{$-0.10$}  & \blue{$\mathbf{+0.10}$} & $40124$ & $44029$ & $45161$ & $51840$ & $1.10\times$ & $1.13\times$ & $1.29\times$ \\
  \midrule
  \multirow{4}{*}{$\mathrm{Gemma4\text{-}31B}$}
    & $\mathrm{AIME}$ & $88.10$ & $89.60$ & $89.50$ & $90.40$ & \blue{$+1.50$} & \blue{$+1.40$} & \blue{$\mathbf{+2.30}$} & $7741$  & $9620$  & $9695$  & $12882$ & $1.24\times$ & $1.25\times$ & $1.66\times$ \\
    & $\mathrm{GPQA}$ & $85.30$ & $85.40$ & $85.40$ & $85.60$ & \blue{$+0.10$} & \blue{$+0.10$} & \blue{$\mathbf{+0.30}$} & $5521$  & $7014$  & $6987$  & $9320$  & $1.27\times$ & $1.27\times$ & $1.69\times$ \\
    & $\mathrm{LCB}$  & $80.80$ & $81.00$ & $81.50$ & $81.90$ & \blue{$+0.20$} & \blue{$+0.70$} & \blue{$\mathbf{+1.10}$} & $6553$  & $8288$  & $8224$  & $10795$ & $1.26\times$ & $1.26\times$ & $1.65\times$ \\
    & $\mathrm{HMMT}$ & $88.10$ & $88.70$ & $88.10$ & $88.90$ & \blue{$+0.60$} & \blue{$+0.00$} & \blue{$\mathbf{+0.80}$} & $8499$  & $10430$ & $10541$ & $13730$ & $1.23\times$ & $1.24\times$ & $1.62\times$ \\
  \midrule
  \multirow{4}{*}{$\mathrm{R1\text{-}Distill\text{-}32B}$}
    & $\mathrm{AIME}$ & $55.60$ & $55.90$ & $56.40$ & $57.50$ & \blue{$+0.30$} & \blue{$+0.80$} & \blue{$\mathbf{+1.90}$} & $11065$ & $11979$ & $12318$ & $15198$ & $1.08\times$ & $1.11\times$ & $1.37\times$ \\
    & $\mathrm{GPQA}$ & $62.20$ & $62.00$ & $62.30$ & $62.80$ & \red{$-0.20$}  & \blue{$+0.10$} & \blue{$\mathbf{+0.60}$} & $6598$  & $7430$  & $7277$  & $8744$  & $1.13\times$ & $1.10\times$ & $1.33\times$ \\
    & $\mathrm{LCB}$  & $25.80$ & $27.90$ & $28.40$ & $35.00$ & \blue{$+2.10$} & \blue{$+2.60$} & \blue{$\mathbf{+9.20}$} & $9859$  & $10990$ & $10715$ & $12526$ & $1.11\times$ & $1.09\times$ & $1.27\times$ \\
    & $\mathrm{HMMT}$ & $43.80$ & $44.10$ & $44.50$ & $45.50$ & \blue{$+0.30$} & \blue{$+0.70$} & \blue{$\mathbf{+1.70}$} & $12912$ & $13867$ & $14271$ & $17278$ & $1.07\times$ & $1.11\times$ & $1.34\times$ \\
  \midrule
  \textbf{$\mathrm{micro\ avg}$} & $-$ & $64.82$ & $65.18$ & $65.47$ & $66.04$ & \blue{$+0.36$} & \blue{$+0.65$} & \blue{$\mathbf{+1.21}$} & $-$ & $-$ & $-$ & $-$ & $-$ & $-$ & $-$ \\
  \textbf{$\mathrm{macro\ avg}$} & $-$ & $62.11$ & $62.72$ & $63.01$ & $63.75$ & \blue{$+0.61$} & \blue{$+0.90$} & \blue{$\mathbf{+1.65}$} & $16891$ & $18987$ & $19255$ & $23032$ & $1.12\times$ & $1.14\times$ & $1.36\times$ \\
  \bottomrule
  \end{tabular}%
  }
\end{table*}

% Color macros (add to preamble):
% \usepackage{xcolor}
% \newcommand{\green}[1]{\textcolor{teal}{#1}}
% \newcommand{\red}[1]{\textcolor{red!70!black}{#1}}
% \newcommand{\gray}[1]{\textcolor{gray}{#1}}
\vspace{-0.5em}
\subsection{Main Results}    
\label{sec:main-results}
\vspace{-0.5em}
Table~\ref{tab:main-results} summarizes the results.
\vspace{-0.5em}
\paragraph{InsightReplay improves accuracy across all $24$ settings.}
IR3 yields strictly non-negative gains on every cell, with macro-averaged improvement of $+1.65$ points and micro-averaged improvement of $+1.21$ points over Base. The largest gains are observed on \texttt{R1-Distill-32B}/LCB ($+9.2$); the smallest are on settings already close to ceiling, such as \texttt{Qwen3.5-35B}/HMMT ($+0.1$).

\vspace{-0.5em}
\paragraph{Gains are not merely an artifact of extended thinking.}
VO instructs the model to keep thinking after its initial trace but performs no insight extraction; despite using comparable token budgets, VO captures only $+0.61$ of the macro-averaged gain. The remaining $+1.04$, captured by IR3 over VO, is attributable to the InsightReplay mechanism itself. In other words, more than $60\%$ of the total improvement comes from insight extraction and replay rather than from additional thinking time.

\vspace{-0.5em}
\paragraph{Token overhead is moderate.}
IR3 consumes on average $1.36\times$ the tokens of Base, reflecting the insight extraction and replay rounds; IR1 incurs only $1.14\times$ overhead while still delivering a macro-averaged improvement of $+0.90$. Given the consistent accuracy gains across $24$ settings, 
many of which involve baselines already near saturation, this represents a favorable 
accuracy--cost trade-off.

\vspace{-0.5em}
\paragraph{Agentic Evaluation.}
We additionally evaluate InsightReplay in an agentic setting on SWE-Bench-Verified~\citep{swebench}. Due to space constraints, we defer the full protocol and results to Appendix~\ref{app:swe-bench-eval}.

\vspace{-0.5em}
\subsection{Empirical Validation of Peak Accuracy}
\label{sec:inverted-u-validation}
\vspace{-0.5em}
We test the prediction of Theorem~\ref{thm:ir_peak}: InsightReplay 
should raise the achievable peak accuracy. We reuse the samples 
collected for the main results. For each problem, we sort the $16$ 
Baseline completions by token count and bin them into $8$ length 
levels (two samples per level). Binning is performed within each 
problem rather than globally to avoid conflating length with 
difficulty: a global short bin would be dominated by easy problems. 
The VO/IR1/IR3 samples inherit the bin assignment of their parent 
Baseline, so all variants are compared within each length level on 
samples sharing the same Baseline starting points. The Baseline 
curve in Figure~\ref{fig:length-sweep-30b} reports per-bin accuracy 
averaged over the four benchmarks ($8$ data points). For each of 
VO, IR1, and IR3, we compute per-bin accuracy under the same 
binning protocol and report the peak across the $8$ bins, 
shown as a single star marker. See details in 
Appendix~\ref{app:length-sweep-protocol}.
To probe whether the 
trend extrapolates beyond the ablated range, we additionally run 
IR5 on the LiveCodeBench v5 subset; the result is 
reported in Figure~\ref{fig:teaser}.
\begin{figure}[ht]
\centering
\includegraphics[width=\textwidth]{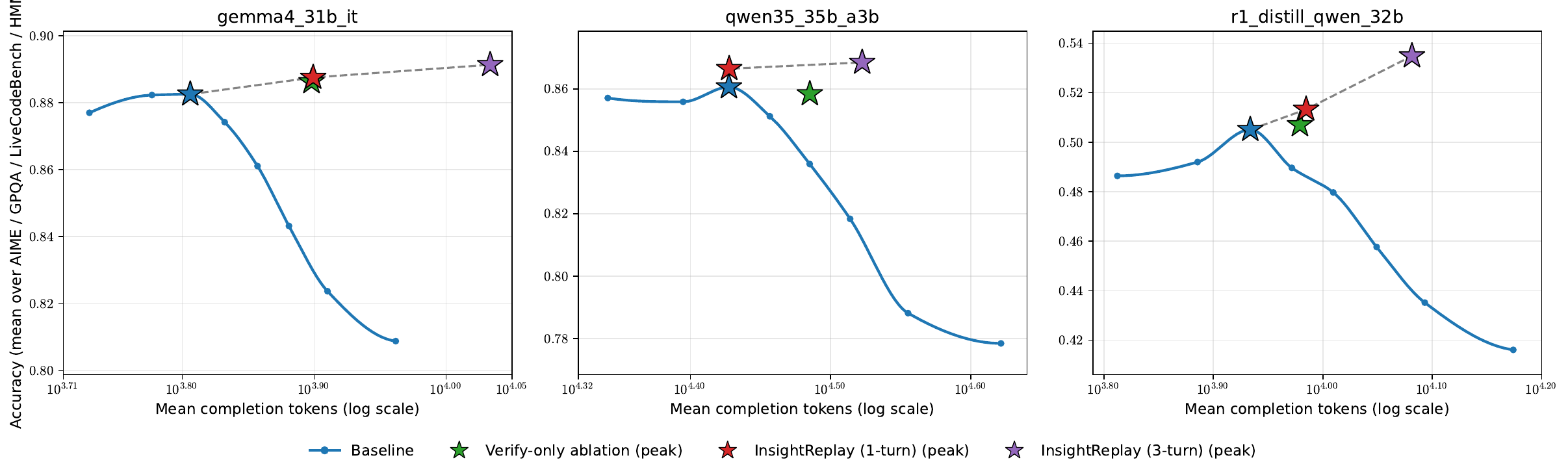}
\caption{Accuracy vs.\ mean completion tokens for three 
$\sim30$B-scale models, averaged over the four benchmarks. The 
Baseline curve traces accuracy across $8$ length levels of standard 
CoT; each star marker reports the peak accuracy of VO, IR1, or IR3 
on the same set of problems. Dashed lines connect the Baseline peak 
to the IR1 and IR3 peaks, visualizing how peak accuracy evolves with the number of replay rounds. Results for 
$\sim8$B-scale models are shown in 
Appendix~\ref{app:length-sweep-protocol} due to space.}

\label{fig:length-sweep-30b}
\end{figure}

\vspace{-0.5em}
Figure~\ref{fig:length-sweep-30b} reports the curves for the three $\sim30$B-scale models. We highlight two findings.

% \paragraph{InsightReplay extends the productive reasoning regime.} 
% Beyond the token budget where the Base curve has already started to 
% decline, IR3 continues to gain accuracy. Concretely, the Base peak 
% occurs at roughly 25K, 6.5K, and 9K tokens on 
% \texttt{Qwen3.5-35B-A3B}, \texttt{Gemma-4-31B-it}, and 
% \texttt{R1-Distill-32B} respectively; past these points, additional 
% reasoning hurts the Base. IR3, in contrast, keeps improving up to 
% roughly 38K, 11K, and 12K tokens on the same three models, 
% converting the Base's degradation regime into a continued-growth 
% regime. This matches Theorem~\ref{thm:ir_shift}'s prediction that 
% mitigating insight decay extends the regime in which longer 
% reasoning remains productive.

\vspace{-0.5em}
\paragraph{The peak accuracy rises.} The IR3 peak exceeds the Base peak on all models: $86.06\%\!\to\!86.83\%$ on \texttt{Qwen3.5-35B-A3B}, $88.25\%\!\to\!89.13\%$ on \texttt{Gemma-4-31B-it}, and $50.51\%\!\to\!53.48\%$ on \texttt{R1-Distill-32B}. This matches Theorem~\ref{thm:ir_peak} that the same mitigation also lifts the achievable peak.

\vspace{-0.5em}
\paragraph{Test-time scaling alone is insufficient.}
Budget-forced thinking alone fails 
to lift the peak. On \texttt{Qwen3.5-35B-A3B}, VO's peak is below 
Base. The peak rise is therefore driven primarily by the 
InsightReplay mechanism rather than extended thinking alone.

\vspace{-0.5em}
\section{Post-Training Experiments}
\label{sec:training}
\vspace{-0.5em}
To study whether InsightReplay can be reinforced through RL, 
we train \texttt{Qwen3-4B-Base}~\citep{qwen3} with GRPO, 
comparing InsightReplay-enabled rollouts against standard CoT 
rollouts. 

\vspace{-0.5em}
\paragraph{Setup.}
We use GRPO~\citep{grpo} as the policy-optimization algorithm and adopt the
DAPO~\citep{dapo} soft overlong reward shaping to stabilize long-CoT
training. The actor is trained on \texttt{DAPO-Math-15K}~\citep{dapo}
with a maximum response length of \(30{,}720\) tokens, a rollout
group size of \(8\) per prompt, and a global batch size of \(256\)
prompts. Validation is performed on AIME~2025 every \(20\) training
steps with \(n{=}32\) samples per problem, from which we report
\(\mathrm{best@32}\), \(\mathrm{maj@32}\), and \(\mathrm{mean@32}\).
Both runs use the same base model, data, and random seed. Full hyperparameters are
provided in Appendix~\ref{app:training-details}.

% \paragraph{Validation Metrics.}
% We use three complementary metrics over the \(32\) samples per problem: 
% \(\mathrm{mean@32}\) (average per-sample accuracy), 
% \(\mathrm{best@32}\) (fraction of problems with at least one correct 
% sample, reflecting upper-tail capability), and \(\mathrm{maj@32}\) 
% (majority vote, reflecting whether the model converges on the correct 
% answer consistently).

\vspace{-0.5em}
\paragraph{Rollout Protocol.}
The two settings differ only in how rollouts are generated. In 
the \textbf{Baseline} setting, the policy 
generates a single thinking trace until it emits the 
EOS token or reaches the maximum response length. 
In the \textbf{InsightReplay} setting, once the policy emits 
the EOS token, we append a fixed continuation cue 
asking the model to extract and replay the insights 
(exact text in Appendix~\ref{app:rollout-cue}). Generation then 
continues under the same length budget, so the total response 
length matches the baseline exactly.
The cue is inserted with $\mathrm{response\_mask}{=}0$ which isolates the effect of the 
InsightReplay reasoning pattern from any auxiliary supervision 
signal.

\vspace{-0.5em}
\paragraph{Alternative Training Designs.}
We also experimented with a more elaborate scheme that explicitly
encodes the multi-round InsightReplay structure during training via
a 2D attention mask combined with per-round trace bundling. While methodologically clean, this design costs more compute than the protocol above. Nonetheless,
we document this attempt in
Appendix~\ref{app:alternative-training-design} for the interested
reader.
\begin{figure*}[ht]
\centering
\includegraphics[width=\textwidth]{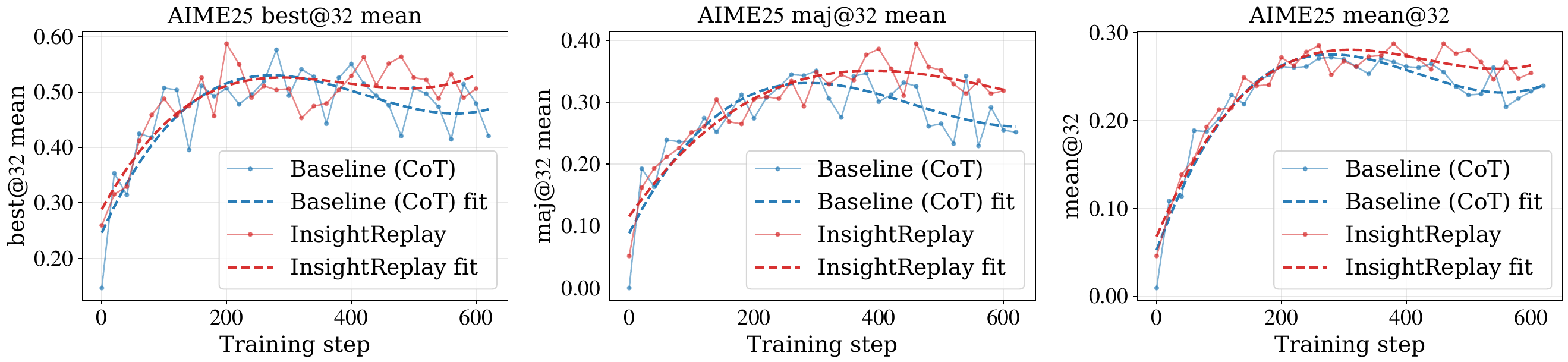}
\caption{Validation accuracy on AIME 2025 throughout GRPO training,
under InsightReplay rollouts vs.\ standard CoT rollouts.
Across all three metrics (\(\mathrm{best@32}\), \(\mathrm{maj@32}\),
\(\mathrm{mean@32}\)). Dashed curves show cubic polynomial trend lines;
the fitting procedure is described in Appendix~\ref{app:training-curve-fit}.}
\label{fig:training-compare-best-mean}
\end{figure*}

\vspace{-0.5em}
\paragraph{Results.} Figure~\ref{fig:training-compare-best-mean} 
reveals three observations comparing InsightReplay against the 
baseline.

\noindent\textbf{(1) Training stability.}
The standard CoT baseline exhibits a pronounced late-training
degradation: after step \(300\), all three metrics decline. The \(\mathrm{mean@32}\) drops from a peak of
\(27.9\%\) at step \(280\) to \(21.6\%\) at step \(560\); \(\mathrm{best@32}\) similarly decays from
\(57.5\%\) at step \(280\) to \(42.0\%\) at step \(620\). In
contrast, InsightReplay maintains stable performance throughout
training: from step \(200\) onward, \(\mathrm{mean@32}\) stays
within the \(25.0\%\)--\(29.0\%\) band, with no monotonic decline
observed up to step \(600\).

\noindent\textbf{(2) Peak accuracy.}\quad
InsightReplay also achieves higher peak validation accuracy across
all three metrics. On \(\mathrm{maj@32}\), InsightReplay peaks at
\(38.8\%\) versus \(34.1\%\) for the baseline, an absolute improvement of \(+4.7\) points; on
\(\mathrm{mean@32}\), \(28.8\%\)versus \(27.9\%\), a gain of \(+0.9\) points; and on
\(\mathrm{best@32}\), \(59.1\%\) versus \(57.5\%\), a gain of \(+1.6\)
points. The peak gap is most pronounced on \(\mathrm{maj@32}\),
which aggregates evidence across the \(32\) samples and is
therefore most sensitive to whether the underlying reasoning is
consistent suggesting that InsightReplay
improves the distribution of outputs.

\noindent\textbf{(3) The advantage is not an initialization effect.}
\quad For the first \(\sim 200\) steps, the two methods' learning 
curves overlap within sampling noise on all three metrics; divergence 
emerges only after step \(\sim 250\). This rules out alternative 
explanations such as the InsightReplay prompt placing the policy in 
a more favorable starting region or its reward shape biasing the 
optimization from the outset. The benefit must be produced during 
training rather than inherited from the setup.

Together, the three observations show that the InsightReplay reasoning 
pattern can be reinforced through RL training rather than only used at 
inference time. Beyond improving peak accuracy, it stabilizes long-CoT 
GRPO training in a regime where the standard CoT baseline suffers a 
pronounced late-training degradation, and does so without relying on 
any early-learning advantage. This positions InsightReplay as a 
reasoning pattern compatible with both inference-time and training-time 
use.

\vspace{-0.5em}
\section{Limitation}
\label{sec:limitations}
\vspace{-0.5em}
Although InsightReplay can deliver substantial absolute gains where
the base model has substantive room to grow, with the largest
single-cell improvement reaching \(+9.2\) points on
\texttt{R1-Distill-32B}'s LiveCodeBench v5 subset (lifting \(25.8\%\) to
\(35.0\%\), \(+35.7\%\) relative), the gains are bounded by each
benchmark's available headroom. On near-saturated settings the
marginal improvement is modest: on \texttt{Qwen3.5-35B-A3B}'s HMMT
subset, IR3 improves over an \(88.5\%\) baseline by only \(+0.1\) points. 
This pattern is consistent with our theoretical framework, in which 
$\Phi(d_0)$ relocation cannot exceed the ceiling imposed by 
$p(N)$ itself. We therefore expect InsightReplay to provide the 
largest practical value on hard reasoning tasks where the underlying 
model is far from saturation.

\vspace{-0.5em}
\section{Conclusion}
\vspace{-0.5em}
We introduced InsightReplay, a stateful reasoning approach that periodically
extracts critical insights from the CoT trace and replays them near the active generation
frontier. By mitigating the attention decay that critical insights suffer in long
trajectories, InsightReplay reshapes the inverted-U between CoT length and
accuracy, shifting the peak rightward and lifting it upward, thereby expanding the
regime in which extended reasoning translates into higher accuracy. We support this
through both theoretical analysis and
empirical validation: positive gains across all $24$ cells of our
scale\,$\times$\,family\,$\times$\,benchmark grid in the pure-inference setting, and
consistent improvements in training stability and peak accuracy when used as a rollout
protocol for post-training.

\newpage
\bibliography{references}

\appendix

\section{Experimental Details for Property~I}
\label{app:property1}

\paragraph{Dataset.}
We use $60$ problems from the American Invitational Mathematics Examination (AIME) spanning the $2024$ and $2025$ competitions.

\paragraph{Thinking trace generation.}
For each problem, we generate $16$ independent chain-of-thought samples using \texttt{Qwen3-8B} in its native thinking mode, and report results averaged over these samples. We use the Qwen3 chat template and prompt the model with:
\begin{quote}
\itshape Solve the following math competition problem step by step. Show your detailed reasoning, then give your final answer as a single integer inside \texttt{\textbackslash boxed\{\}}.
\end{quote}
Generation is performed via vLLM~\citep{kwon2023efficient} with sampling (\texttt{temperature}${}=0.6$, \texttt{top\_p}${}=0.95$), tensor parallelism across $2$ GPUs, \texttt{max\_model\_len}${}=30{,}000$, \texttt{dtype=bfloat16}, and a maximum of $30{,}000$ new tokens. The model's output follows the \texttt{<think>...</think>} convention; we parse the content inside the tags as the thinking trace and the content outside as the final response. For traces that reach the $30{,}000$-token limit before producing a closing \texttt{</think>} tag, we treat the content up to the closing tag as the thinking trace if it appears, and otherwise treat the entire generated output as the thinking trace. All seven conditions in Section~\ref{sec:property1} are evaluated on the same set of $60 \times 16 = 960$ generated traces.

\paragraph{Insight extraction.}
For each thinking trace, we use \texttt{Qwen3-8B} to extract a list of $5$--$7$ key insights. The extraction model receives the original question and the full thinking trace, and is instructed to:
\begin{quote}
\itshape Extract exactly $5$ to $7$ key insights from this thinking, in the order they were derived. Each insight should be one concrete sentence containing specific numbers, equations, or results---not vague descriptions.
\end{quote}
The model produces a numbered list (e.g.\ \textit{1.\ The area of $\triangle ABC$ is $6\sqrt{14}$ by Heron's formula.}), which we parse with a regex-based extractor that captures each enumerated entry. We cap the result at $7$ insights per problem in the rare cases where the model produces more. This phase uses vLLM with greedy decoding, \texttt{max\_tokens}${}=8192$, and the same tensor parallelism configuration as the trace generation phase. Across all $60$ problems and $16$ samples per problem, the mean number of extracted insights is $6.9$, with all extractions falling within the requested $5$--$7$ range.

\paragraph{Probability scoring.}
We load \texttt{Qwen3-8B} via HuggingFace Transformers in \texttt{bfloat16} precision on a single GPU.
For each problem, we construct a prefix string following the Qwen3 chat template:
\begin{verbatim}
<|im_start|>user
{user_prompt}<|im_end|>
<|im_start|>assistant
<think>
{THINK_CONTENT}
</think>

\boxed{
\end{verbatim}
where \texttt{THINK\_CONTENT} varies across conditions (see below).
We then append the ground-truth answer tokens $a_{\mathrm{gt}} = (a_1, a_2, \ldots, a_K)$ (the tokenized integer answer) to the prefix and perform a single forward pass over the concatenated sequence. The score is the joint log-probability of the answer tokens conditioned on the prefix:
\begin{equation}
  \log P_\theta(a_{\mathrm{gt}} \mid x_{\mathrm{prefix}})
  = \sum_{k=1}^{K} \log P_\theta\!\bigl(a_k \mid x_{\mathrm{prefix}}, a_{<k}\bigr),
\end{equation}
where $x_{\mathrm{prefix}}$ denotes the prompt up to and including the opening token sequence \texttt{\textbackslash boxed\{}, and $K$ is the number of tokens in the tokenized ground-truth answer.

\paragraph{Condition definitions.}
Let $\mathbf{q}$ denote the problem statement, $\mathbf{t}$ the full thinking trace generated in the previous paragraph, and $\mathbf{f} = (f_1, \dots, f_K)$ the insights extracted by the extraction step (formatted as a numbered list). Let $\mathbf{r}$ denote a sequence of pseudo-random tokens whose token count matches that of $\mathbf{t}$ for the same problem. The seven conditions differ only in \texttt{THINK\_CONTENT}:
\begin{itemize}
  \item[\textbf{(A)}] \textbf{Full thinking (baseline)}: $\mathbf{t}$
  \item[\textbf{(B)}] \textbf{Repeated question + insights}: \texttt{Question:}\,$\mathbf{q}$, followed by $\mathbf{f}$
  \item[\textbf{(C)}] \textbf{Insights only}: $\mathbf{f}$
  \item[\textbf{(D)}] \textbf{Full thinking + repeated question + insights}: $\mathbf{t}$, then \texttt{Question:}\,$\mathbf{q}$, followed by $\mathbf{f}$
  \item[\textbf{(E)}] \textbf{Full thinking + insights}: $\mathbf{t}$, followed by $\mathbf{f}$
  \item[\textbf{(F)}] \textbf{Random tokens + repeated question + insights}: $\mathbf{r}$, then \texttt{Question:}\,$\mathbf{q}$, followed by $\mathbf{f}$
  \item[\textbf{(G)}] \textbf{Random tokens + insights}: $\mathbf{r}$, followed by $\mathbf{f}$
\end{itemize}
Adjacent components are concatenated with two newline characters (\texttt{\textbackslash n\textbackslash n}). To construct $\mathbf{r}$, we sample strings of the form \texttt{xq\{0--9999\}} with a fixed random seed of $42$ for reproducibility, then truncate or pad at the token level to match $|\mathrm{tokenize}(\mathbf{t})|$ exactly.

\paragraph{Software and hardware.}
All experiments use vLLM $0.19.0$ for trace generation and insight extraction, and HuggingFace Transformers $5.6.0$ for probability scoring. Models are served in \texttt{bfloat16} precision. Trace generation and insight extraction each run on $2{\times}$ NVIDIA GPUs with tensor parallelism; probability scoring runs on a single GPU.

\section{Experimental Details for Property~II}
\label{app:property2}

\paragraph{Models.}
\texttt{Qwen3-8B}: 36 layers, 32 query heads, 8 KV heads via GQA, head dim $128$, full RoPE, with QK-norm. \texttt{Gemma-4-31B-it}: alternates sliding-window and global-attention layers with different head dimensions and per-layer \texttt{partial\_rotary\_factor}; we extract scores only from global layers (head\_dim $=512$) so the attention pattern is not truncated by the sliding window. The insight-extraction model is \texttt{Qwen3.5-397B-A17B}.

\paragraph{CoT generation.}
Each target model generates a complete CoT trace via vLLM with greedy decoding (\texttt{max\_new\_tokens}$=$\texttt{38{,}000}, \texttt{temperature}$=$\texttt{1.0}, \texttt{do\_sample}$=$\texttt{False}, \texttt{bfloat16}). The user prompt is:
\begin{quote}
\itshape Solve the following math problem step by step. Show your complete reasoning, then give the final answer inside \texttt{\textbackslash boxed\{\}}.
\end{quote}
For Qwen3-8B we wrap the trace in \verb|<think>...</think>|; for Gemma-4 we use the native thought block. The ground-truth answer is appended as \verb|\boxed{<answer>}| at the fixed position $p_a$ (the first token after \verb|\boxed{|).

\paragraph{Insight extraction and token-level localization.}
We prompt \texttt{Qwen3.5-397B-A17B} to list $5$--$7$ critical insights, each containing a verbatim span copied from the trace. Each insight is then relocated in the full tokenized sequence via a case-sensitive substring match plus character-to-token offset mapping using the target model's tokenizer, yielding \texttt{(start\_token\_idx, end\_token\_idx)} for the insight span $c$. Insights whose verbatim span cannot be recovered, or whose sentence position leaves fewer than one downstream sentence before $p_a$, are discarded. This yields an average of ${\sim}6$ usable insights per problem.

\paragraph{Filler unit calibration.}
The filler unit \verb|"\n "| is calibrated per tokenizer at runtime to determine its token cost $\tau$. The number of filler units actually inserted at ratio $\rho$ is $N = \lfloor \lfloor \rho \cdot d \rfloor / \tau \rfloor$, and the realized distance after insertion is $d' = d + N \cdot \tau$. Random allocation of fillers across sentence boundaries uses a seeded RNG (\texttt{seed}$=$\texttt{42} plus problem-id hash) for reproducibility.

\paragraph{Attention-score extraction.}
We run a layer-by-layer forward pass with \texttt{bfloat16}. At each layer $\ell$ we recompute the post-RoPE query vector $\mathbf{q}^{(\ell,h)}_{p_a}$ and post-RoPE key vectors $\mathbf{k}^{(\ell,h)}_{t}$ for $t \in c$, and store the per-layer score $s^{(\ell)}(c, a)$ together with the realized distance $d'$. For GQA models we map each query head to its shared KV head via integer division by the GQA group size.

\paragraph{Model-specific handling.}
\texttt{Qwen3-8B} applies standard full-dimension RoPE and QK-norm; we replicate the ordering q-proj $\to$ q-norm $\to$ RoPE, and likewise for keys. \texttt{Gemma-4-31B-it} uses head\_dim $=512$ on global layers; to avoid materializing the full attention matrix at this size we use PyTorch SDPA for \texttt{head\_dim}$>$\texttt{256} and fall back to \texttt{flash\_attention\_2} elsewhere. Only the explicit $\mathbf{q}\cdot\mathbf{k}$ computation for the (answer, insight) pair enters the logging path.

\paragraph{Aggregation and reporting.}
Scores are aggregated per $(\ell, \rho)$ by first averaging across samples, then across insights within a problem, then across problems. For distance-based plots, we bin by realized distance $d'$ with bin size $500$ tokens. We additionally identify candidate ``receiver'' heads as the top-$32$ heads ranked by score kurtosis across variants.

\paragraph{Compute.}
Attention extraction uses $8$-way data parallelism across NVIDIA H200 GPUs. Each worker loads one full model instance (\texttt{Qwen3-8B} on one GPU; \texttt{Gemma-4-31B-it} sharded across multiple GPUs via \texttt{device\_map}). All intermediate artifacts (CoTs, insights, token locations, variants, raw scores) are persisted to JSONL/CSV and released with the code.

\section{Proofs of Theoretical Results}
\label{app:proofs}

\subsection{Preparation}
%==============================================================================

Define the baseline log-accuracy term
\begin{equation}
    \ln A_0(N) := N \ln p(N), \qquad g(N) := \frac{d}{dN}\bigl[N \ln p(N)\bigr].
    \label{eq:gdef}
\end{equation}

\subsubsection{Proof of Lemma~\ref{lem:concave}}

\begin{proof}
Let $c:=T/M$. From Eq.~\eqref{eq:A0}, we can write
\begin{equation}
    \ln A_0(N) = \ln \alpha + N\ln(1-\sigma) + N\ln\!\left(1-\frac{c}{N}\right).
\end{equation}
The first term is constant and the second is linear in $N$, so it suffices to study
\begin{equation}
    h(N) := N\ln\!\left(1-\frac{c}{N}\right), \quad
    h'(N) = \ln\!\left(1-\frac{c}{N}\right) + \frac{c}{N-c},  \quad
    h''(N) = -\frac{c^2}{N(N-c)^2} < 0.
\end{equation}
% Computing the second derivative,
% \begin{equation}
%     h''(N) = \frac{c}{N(N-c)} - \frac{c}{(N-c)^2} = -\frac{c^2}{N(N-c)^2} < 0 \qquad \text{for all } N>c.
% \end{equation}
Hence $(\ln A_0)''(N)<0$ for all $N>T/M$, proving strict concavity. Since $g(N)=(\ln A_0)'(N)$, it follows immediately that
\begin{equation}
    g'(N) = (\ln A_0)''(N) < 0 \qquad \text{for all } N>T/M.
    \label{eq:g_decreasing}
\end{equation}
When $N \leq T/M$, we have $E=T/(NM) \geq 1$, which would imply a per-step error rate of at least $100\%$ and is therefore outside the meaningful regime of the model.
\end{proof}

% \begin{lemma}[Uniqueness of Optima]
% \label{lem:unique}
% $A_\Phi$ and $A_{\mathrm{IR}}$ each have a unique global maximum on $(0,\infty)$.
% \end{lemma}

\subsubsection{Proof of Lemma~\ref{lem:unique}}

\begin{proof}
For analytical convenience, we approximate $\sum_{i=1}^{N}\ln\Phi(i)$ by $\int_1^N \ln\Phi(x)\,dx$, so that
\begin{equation}
    \frac{d}{dN}\sum_{i=1}^{N}\ln\Phi(i) \;\approx\; \ln\Phi(N).
    \label{eq:sum_approx}
\end{equation}
Taking logarithms of Eq.~\eqref{eq:A_both},
\begin{equation}
    \ln A_\Phi(N) = \ln A_0(N) + \sum_{i=1}^{N}\ln\Phi(i), \qquad
    \ln A_{\mathrm{IR}}(N) = \ln A_0(N) + N\ln\Phi(d_0).
\end{equation}
Differentiating and applying Eq.~\eqref{eq:sum_approx} and Eq.~\eqref{eq:gdef},
\begin{equation}
    (\ln A_\Phi)'(N) = g(N) + \ln \Phi(N), \qquad
    (\ln A_{\mathrm{IR}})'(N) = g(N) + \ln \Phi(d_0).
    \label{eq:deriv_both}
\end{equation}
Setting these to zero gives the optimality conditions
\begin{equation}
    g(N_\Phi^*) = -\ln \Phi(N_\Phi^*), \qquad g(N_{\mathrm{IR}}^*) = -\ln \Phi(d_0).
    \label{eq:opt_cond}
\end{equation}
We prove the claim for $A_\Phi$; the argument for $A_{\mathrm{IR}}$ is analogous. On the meaningful domain $N > T/M$, as $N\to (T/M)^+$, $g(N)\to +\infty$ while $\ln\Phi(N)$ remains finite, so $(\ln A_\Phi)'(N) = g(N) + \ln\Phi(N) \to +\infty > 0$. As $N\to\infty$, $g(N)\to \ln(1-\sigma)<0$ and $\ln\Phi(N)\to -\infty$, so $(\ln A_\Phi)'(N)\to -\infty < 0$. By the intermediate value theorem, a critical point exists. Since $g(N)$ is strictly decreasing (Lemma~\ref{lem:concave}) and $-\ln\Phi(N)$ is strictly increasing (Assumption~\ref{assump:decay}), the optimality condition $g(N)=-\ln\Phi(N)$ has at most one solution, so the maximizer is unique.
\end{proof}

%==============================================================================
\subsection{Main Results}
%==============================================================================

\subsubsection{Proof of Theorem~\ref{thm:ir_shift}}

\begin{proof}
By design, InsightReplay places the relocated insight adjacent to the active generation frontier, so $d_0 \ll N_\Phi^*$ in any meaningful regime. Combined with the strict monotonicity of $\Phi$ (Assumption~\ref{assump:decay}), this yields $\Phi(d_0) > \Phi(N_\Phi^*)$, and hence $-\ln\Phi(d_0) < -\ln\Phi(N_\Phi^*)$. By the optimality conditions in Eq.~\eqref{eq:opt_cond}, $g(N_\Phi^*) = -\ln\Phi(N_\Phi^*)$ and $g(N_{\mathrm{IR}}^*) = -\ln\Phi(d_0)$, so $g(N_{\mathrm{IR}}^*) < g(N_\Phi^*)$. Since $g$ is strictly decreasing (Eq.~\eqref{eq:g_decreasing}), this implies $N_{\mathrm{IR}}^* > N_\Phi^*$.
\end{proof}

\subsubsection{Proof of Theorem~\ref{thm:ir_peak}}

\begin{proof}
Assume $d_0 < 1$ (InsightReplay places the insight within less than one step's distance).
Then for all $i \geq 1$, since $\Phi$ is strictly decreasing (Assumption~\ref{assump:decay}, condition~2), $\Phi(d_0) > \Phi(i)$.
Therefore, for any $N \geq 1$:
\begin{equation}
    A_{\mathrm{IR}}(N)
    = [p(N)]^N \cdot [\Phi(d_0)]^N
    > [p(N)]^N \cdot \prod_{i=1}^{N} \Phi(i)
    = A_\Phi(N).
\end{equation}
In particular, at $N = N_\Phi^*$: $A_{\mathrm{IR}}(N_\Phi^*) > A_\Phi(N_\Phi^*)$.
Since $N_{\mathrm{IR}}^*$ maximizes $A_{\mathrm{IR}}$:
\begin{equation}
    A_{\mathrm{IR}}(N_{\mathrm{IR}}^*)
    \geq A_{\mathrm{IR}}(N_\Phi^*)
    > A_\Phi(N_\Phi^*).
\end{equation}
\end{proof}

\
\section{Pure Inference Experiment Details}
\label{app:experiment-details}

This appendix expands on the experimental setup of Section~\ref{sec:experiments}, covering inference infrastructure, decoding configuration, prompting templates, evaluation protocol, and reproducibility notes that were omitted from the main text for brevity.

\subsection{Hardware and Inference Stack}
All inference experiments are run on a single node with $8\times$ NVIDIA H200 141GB GPUs interconnected via NVLink. Models are 
served through vLLM~\citep{kwon2023efficient} v$0.19.0$ via its OpenAI-compatible HTTP endpoint, with one server instance per model and all $8$ GPUs  
visible to that instance. The parallelism layout depends on model size: the 30B-tier models (Qwen3.5-35B-A3B, Gemma-4-31B-it,
R1-Distill-Qwen-32B) use tensor parallelism $\text{TP}=2$ with data parallelism $\text{DP}=4$, while the 8B-tier models        
(Qwen3.5-9B, Gemma-4-E4B-it, R1-Distill-Qwen-7B) use $\text{TP}=1$ with $\text{DP}=8$. KV-cache memory utilization is capped at
$0.92$. The maximum context length is set per model family at the largest value the KV budget can hold under the chosen
$\text{DP}$: $200{,}000$ tokens for Qwen and Gemma, and the native $131{,}072$ tokens for the R1-distill family. We pass
\texttt{--trust-remote-code} so that custom model code -- e.g.\ Gemma's interleaved local/global attention with sliding-window
size $1024$ -- loads with its released configuration; no custom attention modification is applied on our side.

\subsection{Decoding Configuration}
We sample with temperature $\tau=1.0$ and nucleus $\text{top-}p=0.95$; all other vLLM sampling parameters (e.g.\               
\texttt{top\_k}, \texttt{repetition\_penalty}, presence/frequency penalties) are left at their library defaults so the         
comparison is not biased by ad-hoc decoding tweaks. We pass \texttt{skip\_special\_tokens=False} so that thinking-block markers
($\langle$\texttt{think}$\rangle$\allowbreak\,/\,\allowbreak$\langle$\texttt{/think}$\rangle$ for Qwen and R1-distill,
$\langle|$\texttt{channel}$\rangle$\texttt{thought}\allowbreak\,/\,\allowbreak$\langle$\texttt{channel}$|\rangle$ for Gemma) appear in the raw
response; downstream extraction relies on these markers to split the thinking trace from the post-think tail. The $k=16$
samples per problem are produced by issuing $k$ independent requests with $n=1$ rather than a single request with $n=k$: this
guarantees that vLLM's continuous-batching scheduler treats each sample as an independent unit, and that an HTTP retry on a
transient error does not invalidate sibling samples. Independence across samples is provided by vLLM's internal RNG; no
per-sample seed is pinned.

\subsection{Prompting Templates}
\paragraph{User-turn header.} The user-turn content is a dataset-specific header concatenated with the formatted question; the 
chat template then wraps it into a model-specific prompt. The header explicitly specifies the answer-format the grader keys on:
\begin{itemize}[leftmargin=*,itemsep=2pt]                                                                                      
\item \textbf{AIME}: ``\emph{Solve the math problem. Final answer must be a single non-negative integer in                   
$\langle$\texttt{Answer}$\rangle$\allowbreak\dots\allowbreak$\langle$\texttt{/Answer}$\rangle$.}''                                                   
\item \textbf{GPQA Diamond}: ``\emph{Select the best answer. Final answer must be a single letter (A, B, C, or D) in
$\langle$\texttt{Answer}$\rangle$\allowbreak\dots\allowbreak$\langle$\texttt{/Answer}$\rangle$.}'' 
\item \textbf{LiveCodeBench v5}: ``\emph{Solve the coding task. Put your final solution in a single Python code block
delimited by \texttt{\textasciigrave\textasciigrave\textasciigrave python ... \textasciigrave\textasciigrave\textasciigrave}.}'' 
\item \textbf{HMMT (Feb \& Nov 2025)}: ``\emph{Solve the math problem. Final answer must be the EXACT value (simplified, no  
decimal approximations) in \texttt{\textbackslash boxed\{...\}}. If there are multiple solutions, list all of them inside the  
single \texttt{\textbackslash boxed\{...\}} separated by commas.}'' The boxed-LaTeX format is used (rather than the
$\langle$\texttt{Answer}$\rangle$ tag of AIME/GPQA) because HMMT answers are arbitrary symbolic expressions (e.g.\             
\texttt{\textbackslash frac\{1\}\{2\}}, \texttt{8\textbackslash sqrt\{10\}}) that are graded by symbolic verification rather
than string equality.
\end{itemize}
                                                      
\paragraph{InsightReplay injection.} InsightReplay is a single post-hoc edit of the baseline response: once the baseline emits the thinking-close marker, we (i) ask the model
to enumerate the important conclusions of its own thinking (ii) splice a self-verification block back in \emph{before} the
thinking-close marker; and (iii) let the model continue from that point until it re-emits a thinking-close marker and writes a
fresh final answer. The block is phrased in the model's first person so that, embedded inside the thinking block, it reads as a
continuation of the internal monologue rather than an external instruction. Concretely the block contains:
\begin{itemize}[leftmargin=*,itemsep=2pt]
\item an opening line ``\emph{Wait, before I commit to a final answer, let me restate what's being asked and double-check the
key conclusions I've been relying on}'';                                                 
\item a one-line restatement of the user's header and the question, so the model does not drift from the task in a long      
thinking stream;                                       
\item a numbered list of the extracted insights under ``\emph{Key conclusions so far:}                                 
\end{itemize}                                             

\subsection{Datasets and Preprocessing}                
We use four benchmarks. \textbf{AIME 2025/2026} concatenates the official problem sets of both years, yielding $60$            
short-answer problems ($30$ per year) with non-negative integer answers; each record carries a \texttt{<year>-<idx>} identifier
and the gold answer string. \textbf{GPQA Diamond} is the full $198$-question Diamond split; the upstream release already      
embeds the four answer choices into each problem's \texttt{question} field in their original order (mixing ``A.\ \dots'', ``A)\
\dots'', and ``a)\ \dots'' styles across problems), with the structured \texttt{choices} field left null, so we do not shuffle
the options and we do not append choice lines on our side -- the model sees the question verbatim as released.
\textbf{LiveCodeBench v5} is a fixed $62$-problem uniform subsample of the full $880$-problem v5 release, drawn once to bound
the compute cost of running 30B-class models with $k=16$ samples and ${\sim}20{,}000$-token CoT lengths; the subsample is
shipped with the data files, so all models see the identical $62$-problem subset, with the official hidden test cases attached
verbatim. \textbf{HMMT (Feb \& Nov 2025)} concatenates the MathArena releases of HMMT February 2025 and HMMT November 2025,
yielding $60$ short-answer problems ($30$ per contest, sources \texttt{MathArena/hmmt\_feb\_2025} and
\texttt{MathArena/hmmt\_nov\_2025}); unlike AIME, HMMT gold answers are arbitrary LaTeX expressions (rationals, radicals, exact
closed forms) rather than integers in $[0,999]$, so we keep them as raw strings and defer simplification to the symbolic
grader (Section~\ref{app:eval-protocol}). We apply no few-shot exemplars, no rephrasing of the math/science text, and no
rewriting of the code prompts.

\subsection{Evaluation Protocol}
\label{app:eval-protocol}
The per-sample \texttt{correct} flag is computed by first stripping the model's thinking block (everything up to and including 
the last thinking-close marker) and then applying dataset-specific extraction:
\begin{itemize}[leftmargin=*,itemsep=2pt]                                                                                      
\item \textbf{AIME.} The post-think tail is searched for an integer answer using an ordered list of patterns:
$\langle$\texttt{Answer}$\rangle\,N\,\langle$\texttt{/Answer}$\rangle$, \texttt{\textbackslash boxed\{N\}},                    
\texttt{\textbackslash Answer\{N\}}, ``\texttt{Answer: N}'' (with markdown bold/italic tolerated around the word
``\emph{Answer}'' and between the colon and the digit), and a final ``open $\langle$\texttt{Answer}$\rangle$ without a closing 
tag'' fallback. The first pattern that matches wins; ties within a pattern resolve to the last occurrence. Grading is exact
integer comparison against the gold answer.
\item \textbf{GPQA Diamond.} The same priority list adapted to single-letter answers \texttt{A}--\texttt{D}. We deliberately
do \emph{not} match ``\texttt{Option A/B/C/D}'' phrases inside the reasoning, which would otherwise silently override the      
model's actual final answer with the last-mentioned option. Grading is case-insensitive letter equality.
\item \textbf{LiveCodeBench v5.} The last \texttt{\textasciigrave\textasciigrave\textasciigrave python ... \textasciigrave\textasciigrave\textasciigrave} block in the post-think tail is taken as the candidate  
solution and executed in a Docker sandbox (\texttt{python:3.10-slim}, \texttt{--network none}, \texttt{--cpus 1},              
\texttt{--memory 2g}) against the full hidden test suite. The runner walks all test cases serially within a single sandbox
invocation under a $20$-second wall-clock budget covering all tests for that candidate; a sample is correct iff every test     
passes. We grade $8$ candidates in parallel via a process pool. Function-style tests call \texttt{Solution().<func>(*args)};
stdin-style tests pipe the input to \texttt{python solution.py} and compare stdout (after stripping leading/trailing
whitespace) against the expected output.
\item \textbf{HMMT.} A LaTeX-aware grader is used because gold answers are arbitrary symbolic expressions. We first extract
the candidate's final answer from the post-think tail with a brace-counting scan over the \emph{last} \texttt{\textbackslash   
boxed\{...\}} (brace counting handles nested \texttt{\{\}} that a flat regex would mis-balance), with end-of-text fallbacks to
\texttt{\$\$\,...\,\$\$}, \texttt{\textbackslash[\,...\,\textbackslash]}, inline \texttt{\$\,...\,\$} (with a leading          
``\textit{lhs}\,$=$'' stripped), and finally a bare standalone final line. All fallbacks are restricted to the end of the
response so math fragments mid-reasoning cannot win. Equivalence is checked with the \texttt{math\_verify} library (HuggingFace
MathArena) configured with \texttt{float\_rounding=3} (${\sim}5{\times}10^{-4}$ tolerance) plus two robustness fallbacks: (i)
if the candidate contains \texttt{\textbackslash pm}, we expand it into both signed branches and additionally retry the
comma-joined two-solution form, since the prompt asks for multi-solution sets to be packed into a single \texttt{\textbackslash
boxed\{\}} separated by commas; (ii) if symbolic verification fails, we evaluate gold and candidate to $30$-digit
\texttt{sympy} floats and accept on relative tolerance $10^{-3}$, catching cases where a model truncates a decimal just inside
the symbolic-rounding cutoff. This setup judges \texttt{\textbackslash sqrt\{640\}} equivalent to \texttt{8\textbackslash
sqrt\{10\}} and \texttt{0.5} equivalent to \texttt{\textbackslash frac\{1\}\{2\}}, while keeping \texttt{0.6} and
\texttt{\textbackslash frac\{1\}\{2\}} distinct.
\end{itemize}

\subsection{Case Analysis}
\label{app:case-analysis}

This subsection illustrates the InsightReplay inference pipeline (one extraction
+ one verification round) on a concrete example from our evaluation runs.

\paragraph{Problem.}

\begin{mdframed}[backgroundcolor=gray!10]
Let $A$ be the set of positive integer divisors of $2025$. Let $B$ be a
randomly selected subset of $A$. The probability that $B$ is a nonempty
set with the property that the least common multiple of its elements is
$2025$ is $\frac{m}{n}$, where $m$ and $n$ are relatively prime positive
integers. Find $m + n$.
\end{mdframed}

(AIME 2025, Problem 22; correct answer is $237$. We use \textsc{gemma-4-e4b-it} for this example.)

\paragraph{Color key.}
Across the three passes below, the visual coding is consistent:
\colorbox{cyan!10}{\textbf{cyan}} blocks are model \emph{inputs} (the
prompts the model is given), \colorbox{blue!8}{\textbf{blue}} blocks are
model \emph{outputs} (the generations we read back), the
\colorbox{orange!25}{\textbf{orange}} sub-block inside Pass~3's input
marks the portion that is \emph{spliced in by the inference harness}
rather than written by the model, and any
\colorbox{red!8}{\textit{red italic}} text is an \emph{editorial
annotation by us} --- elisions of long verbatim content, computed
parameter values, or explanatory notes --- that is \emph{not} part of
the literal prompt or output.

\subsubsection*{Pass 1: Baseline Reasoning}

\paragraph{Input.}

\begin{mdframed}[backgroundcolor=cyan!10]
\texttt{User:} Let $A$ be the set of positive integer divisors of $2025$.
Let $B$ be a randomly selected subset of $A$. The probability that $B$
is a nonempty set with the property that the least common multiple of
its elements is $2025$ is $\frac{m}{n}$, where $m$ and $n$ are relatively
prime positive integers. Find $m + n$. The final answer must be a single
non-negative integer wrapped in \texttt{<Answer>...</Answer>}.
\end{mdframed}

\paragraph{Output.}

\begin{mdframed}[backgroundcolor=blue!8]
\texttt{<think>}\\
The set of divisors of $2025 = 3^4 \cdot 5^2$ has size
$|A| = (4+1)(2+1) = 15$, so the total number of subsets is $2^{15}$.
By inclusion--exclusion on the conditions ``no element with $v_3(d)=4$''
($P_3$) and ``no element with $v_5(d)=2$'' ($P_5$):
\begin{align*}
|P_3| &= 2^{12},\quad |P_5| = 2^{10},\quad |P_3 \cap P_5| = 2^8,\\
K &= 2^{15} - 2^{12} - 2^{10} + 2^8 \\
  &= 2^8 \, (2^7 - 2^4 - 2^2 + 1) \\
  &= 256 \, (128 - 16 - 4 + 1) \\
  &= 256 \cdot 105.
\end{align*}
So $P = \frac{256 \cdot 105}{256 \cdot 128} = \frac{105}{128}$,
giving $m=105$, $n=128$, and $m+n = 233$.\\
\texttt{</think>}

\texttt{<Answer>233</Answer>}
\end{mdframed}

The model's reasoning is largely correct, but it commits a subtle
arithmetic slip in the very last simplification:
$128 - 16 - 4 + 1 = 109$, not $105$. The slip is reported
\emph{verbatim} as the final answer, $233$.

For this trace the baseline thinking body is $3326$ tokens, which the
adaptive cap formula $\min(11, 2 + \lfloor (t-1)/5000 \rfloor)$ maps to
$K = 2$ insights; this $K$ flows through into the extraction prompt
below.

\subsubsection*{Pass 2: Extraction}

\paragraph{Input.} The model is given a fresh prompt that wraps its own
\texttt{<think>} body from Pass~1 and asks for a faithful, non-evaluative
listing of the conclusions reached.

\begin{mdframed}[backgroundcolor=cyan!10]
\texttt{User:} You just finished an internal reasoning pass on the
problem below. List the most important concrete conclusions you reached
during that pass --- the facts, equations, intermediate results, or
commitments that your final answer will rest on. These will be fed back
to you as a checklist to verify before you commit.

\texttt{RULES:}
\begin{itemize}[leftmargin=1.5em,itemsep=0pt]
    \item Output AT MOST \colorbox{red!8}{\textit{$K$ (from thinking-body
    length; see Pass~1)}} insights (fewer is fine).
    Pick the most load-bearing.
    \item Each insight: one short sentence, factual and specific.
    \item Do NOT re-derive, do NOT evaluate correctness, do NOT restate the final answer.
    \item Output format: plain bullet list, one insight per line, each line starting with `\texttt{-} '.
\end{itemize}

\texttt{Problem:} \colorbox{red!8}{\textit{(problem restated verbatim --- identical to Pass 1's input)}}

\texttt{Your prior reasoning:}\\
\colorbox{red!8}{\parbox{0.93\linewidth}{\textit{(Pass~1's \texttt{<think>}
body verbatim --- the divisor count, inclusion--exclusion derivation,
and final arithmetic ending at $K = 256 \cdot 105$ and $m + n = 233$.
Elided here for space; in the actual prompt this is several hundred
tokens of literal text from Pass~1's output.)}}}

\texttt{Key conclusions to verify (bullet list only):}
\end{mdframed}

\paragraph{Output.}

\begin{mdframed}[backgroundcolor=blue!8]
\texttt{- } The number of subsets $B \subseteq A$ such that
$\text{lcm}(B) = 2025$ is $K = 2^{15} - 2^{12} - 2^{10} + 2^8 =
256 \cdot 105$.

\texttt{- } The resulting probability is $\frac{m}{n} = \frac{105}{128}$,
where $m$ and $n$ are relatively prime positive integers.
\end{mdframed}

The extractor faithfully surfaces the model's own conclusions ---
\emph{including the arithmetic bug that $256(128-16-4+1) = 256 \cdot 105$}.
This faithfulness is critical: the next pass needs to see the model's
actual reasoning chain, not a sanitized version, otherwise the verification
step has nothing concrete to check.

\subsubsection*{Pass 3: Verification and Commit}

\paragraph{Input.} The model receives a re-thought version of its
original \texttt{<think>} block: the original reasoning from Pass~1 is
preserved verbatim \emph{up to but not including} the closing
\texttt{</think>} marker, and immediately before that marker the
inference harness splices in a structured verification block (highlighted
in orange below). The injection is phrased in the model's own voice so
the continuation reads as a single uninterrupted thinking stream.

\begin{mdframed}[backgroundcolor=cyan!10]
\texttt{User:} \colorbox{red!8}{\textit{(problem statement, identical to Pass~1's input)}}

\texttt{Assistant:} \texttt{<think>}\\
\colorbox{red!8}{\parbox{0.93\linewidth}{\textit{(Pass~1's
\texttt{<think>} body verbatim --- divisor count, inclusion--exclusion
setup, derivation of $K = 256 \cdot 105$, and conclusion $m + n = 233$
--- preserved up to but not including the closing \texttt{</think>}
marker. Elided here for space.)}}}

\begin{mdframed}[backgroundcolor=orange!25,linewidth=0pt,innermargin=2pt,innerleftmargin=6pt,innerrightmargin=6pt,skipabove=4pt,skipbelow=4pt]
\raggedright
\colorbox{red!8}{\textit{[Spliced in by harness]}}

\texttt{Wait, before I commit to a final answer, let me restate what's
being asked and double-check the key conclusions I've been relying on.}

\texttt{The user's request:} Solve the math problem. Final answer must
be a single non-negative integer in \texttt{<Answer>...</Answer>}.

\texttt{Question:} Let $A$ be the set of positive integer divisors of
$2025$ \colorbox{red!8}{\textit{\dots (problem restated verbatim) \dots}} Find $m+n$.

\texttt{Key conclusions so far:}
\begin{enumerate}[leftmargin=1.5em,itemsep=0pt]
    \item The number of subsets $B \subseteq A$ such that
    $\text{lcm}(B) = 2025$ is $K = 2^{15} - 2^{12} - 2^{10} + 2^8 =
    256 \cdot 105$.
    \item The resulting probability is $\frac{m}{n} = \frac{105}{128}$,
    where $m$ and $n$ are relatively prime positive integers.
\end{enumerate}

\texttt{Before finalizing, my current working answer is 233.}

\texttt{Let me verify each of these conclusions and check whether they
actually support this answer --- or whether I've missed something that
would change it.}
\end{mdframed}

\colorbox{red!8}{\textit{(The model now continues from this point,
still inside the same \texttt{<think>} block.)}}
\end{mdframed}

\paragraph{Output.}

\begin{mdframed}[backgroundcolor=blue!8]
Let me re-check each conclusion. The set-up is correct: $|A|=15$,
$|P_3|=2^{12}$, $|P_5|=2^{10}$, $|P_3 \cap P_5|=2^8$. Direct evaluation:
\[
K = 32768 - (4096 + 1024 - 256) = 32768 - 4864 = 27904.
\]
Now I'll re-check the factored form
$K = 2^8 (2^7 - 2^4 - 2^2 + 1)$:
$2^7 - 2^4 - 2^2 + 1 = 128 - 16 - 4 + 1$. Compute step by step:
$128 - 16 = 112$, $112 - 4 = 108$, $108 + 1 = 109$.

\textbf{The previous simplification was erroneous:}
$128 - 16 - 4 + 1 = 109$, not $105$. So $K = 256 \cdot 109 = 27904$
(consistent with the direct evaluation), and
$P = \frac{256 \cdot 109}{2^{15}} = \frac{109}{128}$. Then $m = 109$,
$n = 128$, and $\gcd(m,n) = 1$ since $109$ is prime and $128 = 2^7$.
Therefore $m + n = 109 + 128 = 237$.\\
\texttt{</think>}

\texttt{<Answer>237</Answer>}
\end{mdframed}

The verification pass catches and corrects the arithmetic slip from
Pass~1, flipping the answer from the incorrect $233$ to the correct
$237$. The original reasoning chain was almost entirely sound,
and a single concrete checklist of the model's own intermediate claims
was enough to surface and resolve the error without any external
feedback.

\section{SWE-Bench Evaluation}
\label{app:swe-bench-eval}

To evaluate the effectiveness of InsightReplay on real-world software-engineering tasks, we conduct experiments on SWE-Bench-Verified~\citep{swebench}, a benchmark consisting of GitHub issues paired with corresponding test cases that verify whether the issue has been resolved.

\paragraph{Dataset.}
We randomly select $100$ instances from SWE-Bench-Verified. Each instance contains a natural-language bug report, a Docker image with the target repository checked out at the relevant commit under \texttt{/testbed}, and a set of test cases (both fail-to-pass and pass-to-pass) for evaluation.

\paragraph{Model and Inference Setup.}
We use \textbf{Qwen3.5-35B-A3B} as the base model, served via vLLM~\citep{kwon2023efficient} with $4$-way data parallelism across $4$ GPUs. The vLLM server is configured with a maximum model length of $262{,}143$ tokens, $0.9$ GPU memory utilization, and the Hermes tool-call parser (\texttt{--tool-call-parser hermes}) to enable native function-calling support. For generation, we set temperature to $1.0$, top-$p$ to $0.95$, and the maximum output tokens to $30{,}000$.

\paragraph{Agent Architecture.}
The model operates in an agentic loop with access to six tools: \texttt{file\_read}, \texttt{file\_edit}, \texttt{file\_write}, \texttt{grep}, \texttt{glob}, and \texttt{bash}, all scoped to the \texttt{/testbed} directory inside a Docker container. Each tool is registered as a native function schema in the chat-completion API, allowing the model to invoke tools through structured tool-call messages rather than text-based parsing. The agent iterates for up to $50$ turns per sample (\texttt{max\_turns}=50). Each turn consists of a model generation followed by tool execution. A prompt-guard mechanism monitors the estimated prompt token count and stops the agent early if it approaches the context-window limit of $40{,}960$ tokens (with a $512$-token margin). When the model determines the task is complete, it outputs the token \texttt{DONE} to terminate the loop. The final patch is extracted via \texttt{git diff} from the Docker container.

\paragraph{Sampling and Evaluation.}
For each of the $100$ instances, we independently sample $k=8$ trajectories, yielding $800$ candidate patches in total. Each candidate patch is evaluated by applying it to a fresh Docker container with the original repository state, installing necessary dependencies, and running the instance-specific test suite. A candidate is considered successful if all fail-to-pass tests now pass and all pass-to-pass tests remain passing. We report the \textbf{candidate-level pass rate}, i.e., the fraction of all $800$ candidate patches that resolve the issue.

\paragraph{Compared Methods.}
We compare three prompting strategies, all sharing the same agent architecture, tools, and inference hyperparameters:
\begin{itemize}
    \item \textbf{CoT Zero-shot}: The system prompt instructs the model to think explicitly and concretely before each action and follow a five-step workflow: (1)~understand the bug report, (2)~reproduce the bug, (3)~inspect relevant source files, (4)~analyze the root cause and apply a fix, and (5)~re-run tests to confirm resolution. No few-shot examples are provided.
    \item \textbf{CoT Few-shot}: Extends the CoT zero-shot prompt with three detailed few-shot examples demonstrating best practices: understanding the issue before patching, preferring narrow inspection and minimal edits, and always validating with the original failure mode.
    \item \textbf{InsightReplay (Ours)}: Uses the same five-step workflow but introduces a structured \texttt{<Finding>}\allowbreak\texttt{...}\allowbreak\texttt{</Finding>} mechanism. After each high-level workflow step, the model is required to output exactly one \texttt{<Finding>} block summarizing what was learned in that step. Generation stops at the closing tag, and the accumulated insights are rolled up into a summary message that replaces the growing conversation history. This compresses the context while preserving key insights and enforces a disciplined step-by-step progression. The system prompt also includes a detailed few-shot example illustrating the full five-step InsightReplay workflow.
\end{itemize}

\paragraph{Results.}
Table~\ref{tab:swe-results} reports candidate-level pass rates. InsightReplay achieves $29.75\%$ ($238/800$), outperforming both CoT Zero-shot ($23.12\%$, $185/800$) and CoT Few-shot ($20.87\%$, $167/800$). Relative to the stronger CoT baseline, this is an absolute gain of $6.63$ percentage points and a relative improvement of approximately $28.7\%$.

\begin{table}[t]
\centering
\small
\setlength{\tabcolsep}{4pt}
\caption{Candidate-level pass rates on SWE-Bench-Verified ($100$ selected instances, $k=8$ samples per instance, \texttt{Qwen3.5-35B-A3B}).}
\label{tab:swe-results}
\begin{tabular}{>{\centering\arraybackslash}p{0.24\linewidth}ccc}
\toprule
\textbf{Benchmark} & \textbf{CoT Zero-shot} & \textbf{CoT Few-shot} & \textbf{InsightReplay (Ours)} \\
\midrule
SWE-Bench-Verified & $23.12\%$ ($185/800$) & $20.87\%$ ($167/800$) & $\mathbf{29.75\%}$ ($238/800$) \\
\bottomrule
\end{tabular}
\end{table}

\section{Length-Sweep Protocol Details}
\label{app:length-sweep-protocol}

To jointly characterize the relationship between compute budget and accuracy for the length sweep in Section~\ref{sec:inverted-u-validation}, we apply the following bin-based visualization protocol per (model, dataset, method) triple.

\paragraph{Per-problem binning.}
For each problem, we independently sample $k=16$ baseline completions (temperature $1.0$, top-$p$ $0.95$), producing $16$ thinking trajectories of varying length. We sort these $16$ samples by completion tokens (ascending) and bin them uniformly into $8$ length levels ($2$ samples per level). Crucially, this binning is done \emph{per problem}, not globally across all problems: since problems differ substantially in intrinsic difficulty, a global token-based binning would conflate length with difficulty (a ``short'' bin would be dominated by easy problems, while a ``long'' bin would be dominated by hard problems). Per-problem binning instead identifies, for each problem, its $8$ relative length levels (very-short, short, medium, $\ldots$, very-long), so that within each level we compare samples that come from the \emph{same set of problems}---isolating the length effect from the difficulty effect. Figure~\ref{fig:per-problem-binning} illustrates this procedure schematically.

\begin{figure}[ht]
\centering
\includegraphics[width=0.6\textwidth]{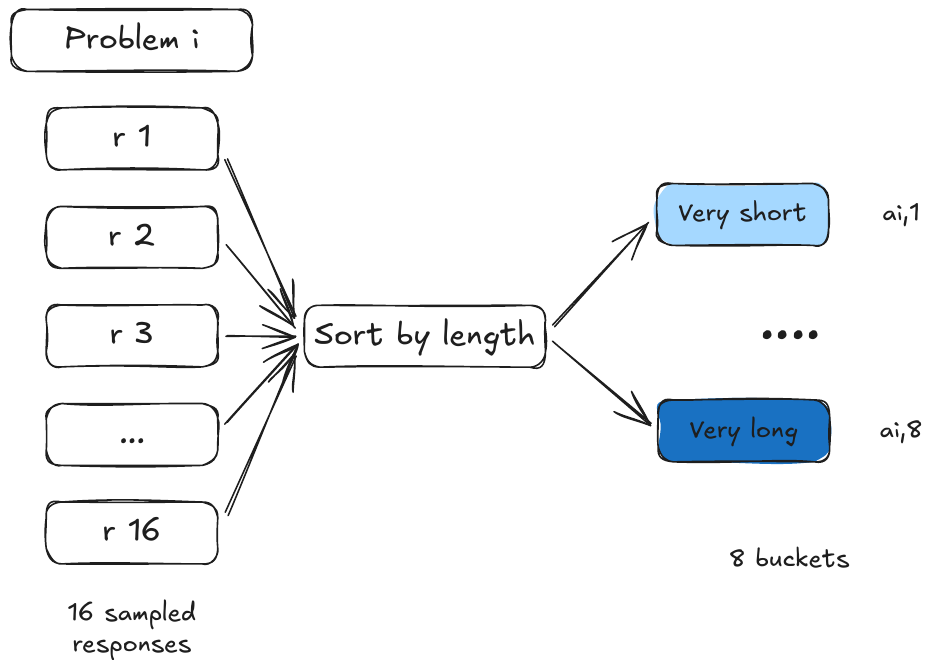}
\caption{%
Per-problem binning procedure. For each problem $i$, we draw $k=16$ independent baseline samples ($r_1, \ldots, r_{16}$), sort them by completion length, and uniformly partition them into $8$ length buckets ($2$ samples per bucket), denoted $a_{i,1}$ (very short) through $a_{i,8}$ (very long). Aggregation across problems then averages within each bucket index, producing the $8$ data points per (model, method) curve in Figure~\ref{fig:length-sweep-30b}. Critically, the binning is done \emph{per problem}: each $a_{i,k}$ refers to problem $i$'s own $k$-th length level relative to its own sampling distribution, not to a global token threshold. This controls for variation in problem difficulty.
}
\label{fig:per-problem-binning}
\end{figure}

\paragraph{Sample-level pairing for VO/IR.}
The corresponding $\mathrm{VO}$ / $\mathrm{IR1}$ / $\mathrm{IR3}$ samples \emph{strictly inherit} the baseline's token ordering: each VO/IR/IR3 sample is paired to its source baseline sample (sorted by that baseline's token count). This pairing ensures that within any given length level, all four methods compare samples extended from \emph{the same set of baseline starting points}, isolating the method's contribution from sampling variance.

\paragraph{Per-level statistics.}
Within each length level, we record two statistics:
\begin{itemize}[leftmargin=*]
\item $x$ = the mean of the $2$ samples' \emph{actual} completion tokens (baseline curve = baseline's own tokens; the other three methods = baseline tokens + injected insights + extended generation tokens);
\item $y$ = the level's accuracy (mean of $\{0,1\}$ correctness).
\end{itemize}

\paragraph{Cross-benchmark aggregation.}
We first compute the $8$ levels' $(x, y)$ values per problem, then average over all problems to obtain the $8$ data points for each (model, dataset, method) triple. For the cross-benchmark merged curves shown in Figures~\ref{fig:length-sweep-30b} and~\ref{fig:length-sweep-8b}, we equally weight the $8$ data points across AIME / GPQA / LiveCodeBench / HMMT (each benchmark contributing $1/4$), to avoid the largest dataset (GPQA, $n = 3168$) dominating the merged curve.

% \paragraph{Peak estimation.}
% Each panel's $4$ dashed colored lines correspond to each method's mean accuracy over the top-$3$ length levels by accuracy---a robust ``peak performance'' estimate (lower variance than the single-point maximum), used to compare ceiling performance across methods.

\paragraph{Results on $\sim8$B-scale models.}
Figure~\ref{fig:length-sweep-8b} reports the length-sweep curves for the three $\sim 8$B-scale models (\mbox{\texttt{Qwen3.5-9B}}, \mbox{\texttt{Gemma-4-E4B}}, and \mbox{\texttt{R1-Distill-Qwen-7B}}), applying the same protocol as in Section~\ref{sec:inverted-u-validation}. The same two patterns observed at the $30$B scale---rightward shift of the optimal CoT length and upward shift of the peak accuracy---hold here as well, indicating that the inverted-U reshaping predicted by Theorems~\ref{thm:ir_shift} and~\ref{thm:ir_peak} is robust across model scales.

\begin{figure}[ht]
\centering
\includegraphics[width=\textwidth]{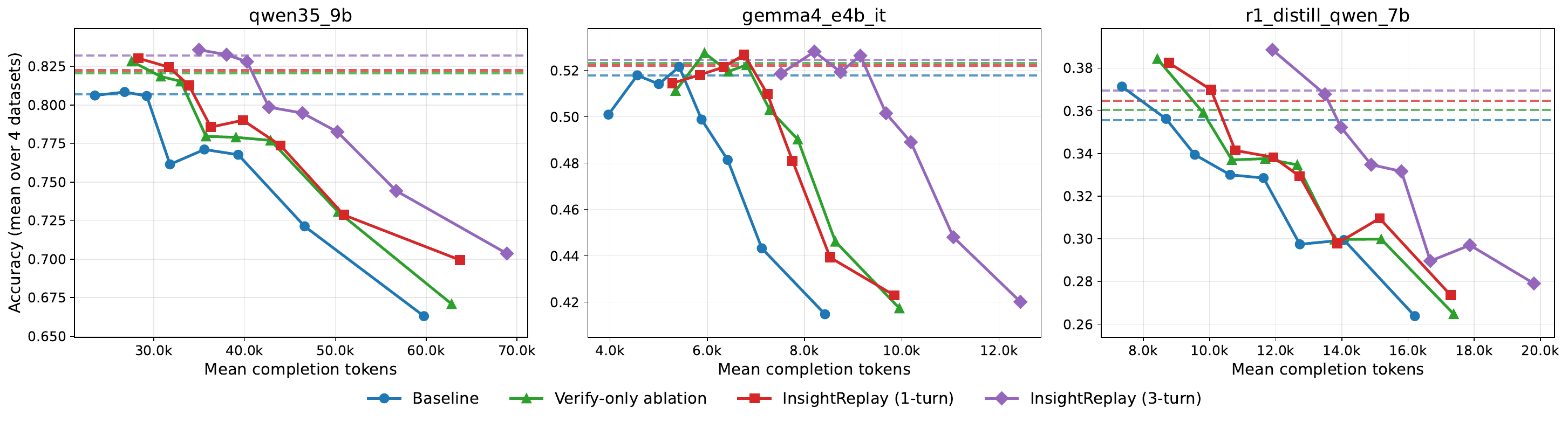}
\caption{%
Length-sweep results for $\sim8$B-scale models (\mbox{\texttt{Qwen3.5-9B}}, \mbox{\texttt{Gemma-4-E4B}}, and \mbox{\texttt{R1-Distill-Qwen-7B}}), averaged with equal weight over the four benchmarks. Dashed lines mark each variant's peak (mean of its top-$3$ length levels). The same rightward and upward shift of the IR3 curve relative to Base observed at the $30$B scale (Figure~\ref{fig:length-sweep-30b}) holds at the $8$B scale, confirming that the inverted-U reshaping is consistent across model sizes.
}
\label{fig:length-sweep-8b}
\end{figure}

\section{Post-Training Details}
\label{app:training-details}

All experiments are implemented on top of the open-source
\textbf{verl} RL training framework~\citep{sheng2024hybridflow}, using
its FSDP2 actor backend and vLLM rollout engine.

This appendix lists the full set of hyperparameters used for the
\textsc{InsightReplay} RL study reported in the main text. Both the
baseline (standard CoT rollout) and the \textsc{InsightReplay} run share
\emph{identical} settings: same base model, same dataset, same optimizer,
same KL/clip configuration, same rollout sampler, same length budget,
and the same compute topology. The two runs differ in exactly one
configuration line — the \textsc{InsightReplay} run additionally sets
\texttt{actor\_rollout\_ref.\allowbreak rollout.\allowbreak agent.\allowbreak default\_agent\_loop=\allowbreak insight\_replay\_agent}, which activates the rollout-time
continuation-cue injection described in Section~\ref{sec:training}; all
remaining values below are shared.

\paragraph{Compute and parallelism.}
\begin{itemize}
\itemsep0pt
  \item Hardware: $16$ nodes $\times$ $8$ GPUs (NVIDIA H800) $=128$ GPUs total.
  \item Training framework: \texttt{verl} with FSDP2; rollout engine: vLLM.
  \item Ulysses sequence-parallel size: $2$; FSDP shard size: $8$;
        rollout tensor-parallel size: $1$ (one vLLM instance per GPU,
        $128$ instances total).
  \item Dynamic batching enabled for both actor update and inference;
        per-GPU token cap $=1\!\times\!(L_{\text{prompt}}+L_{\text{resp}})$
        for the actor and $2\!\times\!(L_{\text{prompt}}+L_{\text{resp}})$
        for inference / log-prob.
  \item No parameter, optimizer, or policy offloading; reshard after
        forward; entropy checkpointing enabled.
  \item {\raggedright vLLM rollout uses chunked prefill with
        \texttt{max\_num\_batched\_tokens}=$32{,}768$,
        \texttt{gpu\_memory\_utilization}=$0.7$,
        prefix caching disabled, and
        \texttt{free\_cache\_engine}=\texttt{True}.\par}
\end{itemize}

\paragraph{Model and data.}
\begin{itemize}
\itemsep0pt
  \item Base model: \texttt{Qwen3-4B-Base} (no SFT warm-up).
  \item Training set: \texttt{DAPO-Math-15K} (RL split).
  \item Validation set: AIME~2025.
  \item Maximum prompt length: $2{,}048$ tokens.
  \item Maximum response length: $30{,}720$ tokens (identical for both
        runs; under \textsc{InsightReplay} the cue and the second-phase
        continuation share this same budget).
  \item Left truncation; data shuffling enabled.
\end{itemize}

\paragraph{Optimization.}
\begin{itemize}
\itemsep0pt
  \item Algorithm: GRPO (\texttt{adv\_estimator=grpo}).
  \item Optimizer: AdamW; learning rate $1\!\times\!10^{-6}$;
        warmup steps $=10$; weight decay $=0.1$; gradient clipping $=1.0$.
  \item Loss aggregation: \texttt{token-mean}.
  \item Entropy coefficient: $0$; entropy computed from logits with
        chunking.
  \item KL regularization \emph{fully disabled}: no KL term in the
        reward (\texttt{use\_kl\_in\_reward=False}, $\beta_{\text{KL}}=0$),
        no KL loss against the reference policy
        (\texttt{use\_kl\_loss=False}, coefficient $=0$).
  \item PPO-style clipping with DAPO's clip-higher asymmetry:
        $\varepsilon_{\text{low}}=0.2$, $\varepsilon_{\text{high}}=0.28$.
  \item No critic warm-up.
\end{itemize}

\paragraph{Rollout protocol.}
\begin{itemize}
\itemsep0pt
  \item Group size per prompt: $n=8$ trajectories.
  \item Global batch: $256$ prompts per training step
        ($\Rightarrow$ $256\times 8=2{,}048$ trajectories per step).
  \item PPO mini-batch: $128$ prompts.
  \item Training sampler: temperature $=1.0$, top-$p=1.0$,
        top-$k=-1$ (disabled);\allowbreak\
        \texttt{ignore\_eos=False}.
  \item Validation sampler: temperature $=1.0$, top-$p=0.7$,
        $n=32$ samples per problem (used for $\mathrm{best@32}$,
        $\mathrm{maj@32}$, $\mathrm{mean@32}$).
\end{itemize}

\paragraph{Reward shaping.}
We use the DAPO reward manager with the soft overlong penalty enabled
to stabilize long-CoT training:
\begin{itemize}
\itemsep0pt
  \item \texttt{overlong\_buffer.enable=True}.
  \item Buffer length: $4{,}096$ tokens (i.e., the soft penalty engages
        once a response is within $4{,}096$ tokens of the
        $30{,}720$-token cap).
  \item Penalty factor: $1.0$.
  \item \texttt{max\_resp\_len} passed to the reward equals the rollout
        \texttt{max\_response\_length} ($30{,}720$).
\end{itemize}

\paragraph{Schedule and checkpointing.}
\begin{itemize}
\itemsep0pt
  \item Total training epochs: $10$ (both runs use the same seed,
        the same data shuffling, and the same number of update steps).
  \item Validation frequency: every $20$ training steps;
        validation is also run before training starts.
  \item Checkpoint frequency: every $50$ training steps;
        checkpoints save \{model, optimizer, extras\}.
  \item \texttt{balance\_batch=False}; \texttt{resume\_mode=auto}.
  \item Logging: TensorBoard.
\end{itemize}

\paragraph{Rollout continuation cue.}
\label{app:rollout-cue}
For the \textsc{InsightReplay} rollout, once the policy emits the end-of-sequence token, we append the following fixed continuation cue before allowing generation to continue under the same response-length budget:

\begin{mdframed}[backgroundcolor=gray!5, linecolor=gray!50!black, nobreak=true]

{\itshape
Wait, before I commit to a final answer, let me restate what's being asked and cross-verify by a completely different method than what I used above.\par
The user's request: \texttt{\{question\}}\par
Key conclusions so far:
}
\end{mdframed}

This cue is inserted with \(\mathrm{response\_mask}{=}0\), so the injected tokens affect the model only as conditioning context and never contribute directly to the policy loss.

\paragraph{InsightReplay-specific configuration.}
The only configuration difference between the two runs is the rollout
agent loop:
\begin{itemize}
\itemsep0pt
  \item Baseline:
        \texttt{default\_agent\_loop} unset (single-pass CoT rollout).
  \item \textsc{InsightReplay}:
        \texttt{actor\_rollout\_ref.rollout.agent.default\_agent\_loop}\\
        \texttt{=insight\_replay\_agent}.
        Whenever phase~1 terminates by emitting EOS, the rollout worker
        splices in the fixed continuation cue (Section~\ref{sec:training})
        and resumes generation; the total response is still capped at
        $30{,}720$ tokens, so the per-trajectory length budget is
        identical to the baseline.
\end{itemize}

A consolidated summary of all hyperparameters appears in
Table~\ref{tab:rl-hparams}.

\begin{table}[ht]
\centering
\small
\caption{Full RL training hyperparameters.
Values are shared by the baseline and \textsc{InsightReplay} runs unless
otherwise noted.}
\label{tab:rl-hparams}
\makebox[\textwidth][c]{%
\begin{tabular}{cc}
\toprule
\textbf{Hyperparameter} & \textbf{Value} \\
\midrule
\multicolumn{2}{c}{\emph{Model \& data}} \\
Base model              & Qwen3-4B-Base (no SFT) \\
Training set            & DAPO-Math-15K \\
Validation set          & AIME~2025 \\
Max prompt length       & $2{,}048$ \\
Max response length     & $30{,}720$ \\
\midrule
\multicolumn{2}{c}{\emph{Algorithm}} \\
RL algorithm            & GRPO \\
Reward manager          & DAPO (soft overlong) \\
KL in reward / KL loss  & disabled / disabled \\
Clip ratio (low, high)  & $(0.2,\ 0.28)$ \\
Entropy coefficient     & $0$ \\
Loss aggregation        & token-mean \\
Overlong buffer length  & $4{,}096$ \\
Overlong penalty factor & $1.0$ \\
\midrule
\multicolumn{2}{c}{\emph{Optimization}} \\
Optimizer               & AdamW \\
Learning rate           & $1\!\times\!10^{-6}$ \\
LR warmup steps         & $10$ \\
Weight decay            & $0.1$ \\
Gradient clipping       & $1.0$ \\
\midrule
\multicolumn{2}{c}{\emph{Batching}} \\
Global batch (prompts)  & $256$ \\
Rollouts per prompt     & $8$ \\
PPO mini-batch (prompts)& $128$ \\
Dynamic batching        & enabled \\
\midrule
\multicolumn{2}{c}{\emph{Rollout (training)}} \\
Sampler                 & temperature $1.0$, top-$p$ $1.0$, top-$k$ off \\
Engine                  & vLLM, chunked prefill, no prefix cache \\
GPU memory utilization  & $0.7$ \\
\midrule
\multicolumn{2}{c}{\emph{Rollout (validation)}} \\
Sampler                 & temperature $1.0$, top-$p$ $0.7$ \\
Samples per problem     & $n=32$ \\
\midrule
\multicolumn{2}{c}{\emph{Schedule}} \\
Total epochs            & $10$ \\
Validation frequency    & every $20$ steps \\
Checkpoint frequency    & every $50$ steps \\
\midrule
\multicolumn{2}{c}{\emph{Compute}} \\
Hardware                & $128\times$ NVIDIA H800 ($16$ nodes) \\
Framework               & verl + FSDP2; vLLM rollout \\
Ulysses SP size         & $2$ \\
FSDP shard size         & $8$ \\
Rollout TP size         & $1$ \\
\midrule
\multicolumn{2}{c}{\emph{Run-specific (only difference)}} \\
\texttt{default\_agent\_loop} & unset (baseline)\,/\,
                                 InsightReplay \\
\bottomrule
\end{tabular}
}
\end{table}

\paragraph{Training-curve smoothing for visualization.}
\label{app:training-curve-fit}
In Figure~\ref{fig:training-compare-best-mean}, for each metric we additionally
overlay a smooth trend line obtained by fitting a degree-$3$ polynomial
\(\hat{f}(t) = \sum_{k=0}^{3} a_k\, t^k\) to the per-step validation values
\((t_i, y_i)\) via ordinary least squares:
\[
  (a_0, a_1, a_2, a_3) \,=\, \arg\min_{a}\;
  \sum_i \bigl(y_i - \hat{f}(t_i)\bigr)^2.
\]
We then evaluate \(\hat{f}\) on $200$ equally spaced points spanning the
observed training-step range and render the resulting curve as a dashed
line. This fit is used purely as a visual smoother: it suppresses
per-evaluation noise so that the relative trend between the two runs is
easier to read at a glance. We do not draw any quantitative conclusion
from the polynomial coefficients themselves.

\section{Alternative Training Method}
\label{app:alternative-training-design}

This appendix documents an earlier attempt to train the multi-round
InsightReplay structure directly, rather than inducing it at
rollout time with a fixed continuation cue. The pipeline has two stages:
we first construct supervised multi-round traces, and then optimize the
policy with a training procedure that explicitly mirrors the round-wise
visibility pattern used at inference time.

\subsection{Method Overview}

\paragraph{Base model and goal.}
We start from \texttt{Qwen3.5-9B-Base}~\citep{qwen35}. The goal of this
alternative design is to make the multi-round reasoning structure
explicit during training: each round exposes only the question, the
currently visible reasoning context, and the replayed intermediate
insights that should be accessible at that round.

\paragraph{Stage 1: Supervised multi-round trace construction.}
To obtain SFT data in the target format, we sample $2{,}000$ math
problems from \texttt{DAPO-Math-17K}~\citep{dapo}. For each problem, a strong
teacher model, \texttt{Qwen3.5-397B-A17B}, generates a full solution,
which we then decompose into a sequence of reasoning steps and compact
intermediate insights. We convert these into multi-round targets where
all non-final rounds contain a reasoning segment together with an
intermediate insight, and the last round produces the final answer. This
procedure yields $1{,}892$ valid problems and $8{,}102$ SFT training
entries. A concrete example appears in Appendix~\ref{appendix:insight_example}.

\paragraph{Stage 2: Direct multi-round RL training.}
Our RL training procedure consists of two steps, illustrated in
Figure~\ref{fig:2d-mask-training}.

In \textbf{Step~1}, we start from the standard causal mask and add a
round-level visibility mask over the concatenated multi-round response.
This makes the intended computation explicit: at each round, the model
can attend only to the original question, the relevant reasoning tokens,
and the replayed insights that should be visible in that round, while
all other tokens are masked out on top of the causal constraint.
However, this alone creates a RoPE mismatch: masked tokens become
invisible under attention, but they still occupy positions during
training, so the relative distance between two visible tokens no longer
matches the rollout setting in which those masked tokens are physically
absent.

In \textbf{Step~2}, we regroup each training case by round. For every
round, we delete all masked-out tokens, keep only the tokens visible in
that round, and then bundle the resulting per-round traces from the same
case for joint optimization. This yields contiguous
\texttt{position\_ids} that match rollout exactly. For the advantage,
standard GRPO uses a sample-wise normalized quantity within a prompt
group, e.g.,
\[
A_i^{\mathrm{GRPO}} =
\frac{R_i - \mathrm{mean}(\{R_j\})}{\mathrm{std}(\{R_j\})}.
\]
In our alternative design, we instead assign the same case-level reward
to every round from the same case and use \(A_r = R - b\) for all
rounds \(r\), reusing the same shared advantage across the bundled
round-level traces.

This construction is methodologically faithful to the intended
multi-round computation, but in practice it is substantially more
expensive than the rollout-time cue-injection protocol used in the main
paper and did not produce comparable gains on AIME~2025.

\begin{figure*}[t]
\centering
\resizebox{\textwidth}{!}{%
\begin{tikzpicture}[>=Latex, scale=0.65, transform shape]

\tikzset{
  cell/.style={draw, line width=0.7pt, minimum size=4.5mm, inner sep=0pt, fill=white},
  on/.style={cell, fill=teal!65},
  tok/.style={draw, line width=0.6pt, minimum height=8mm, minimum width=12mm,
              inner sep=0pt, font=\footnotesize, anchor=center},
  qtok/.style={tok, fill=blue!12},
  ftok/.style={tok, fill=orange!22},
  ttok/.style={tok, fill=teal!18},
}

\matrix (M) [matrix of nodes,
  nodes={cell},
  row sep=-\pgflinewidth, column sep=-\pgflinewidth,
  nodes in empty cells,
  ampersand replacement=\&,
  inner sep=0pt
]{
  \& \& \& \& \& \& \& \& \& \\
  \& \& \& \& \& \& \& \& \& \\
  \& \& \& \& \& \& \& \& \& \\
  \& \& \& \& \& \& \& \& \& \\
  \& \& \& \& \& \& \& \& \& \\
  \& \& \& \& \& \& \& \& \& \\
  \& \& \& \& \& \& \& \& \& \\
  \& \& \& \& \& \& \& \& \& \\
  \& \& \& \& \& \& \& \& \& \\
  \& \& \& \& \& \& \& \& \& \\
};

\foreach \i [count=\k from 1] in {Ques,$F_1$,$T_1$,$T_2$,$F_1$,$F_2$,$T_3$,$F_1$,$F_2$,$F_3$} {
  \node[font=\small, rotate=90, anchor=south] at ($(M-1-\k.north)+(0.18,0.5)$) {\i};
}
\foreach \i [count=\k from 1] in {Ques,$F_1$,$T_1$,$T_2$,$F_1$,$F_2$,$T_3$,$F_1$,$F_2$,$F_3$} {
  \node[font=\small, anchor=east] at ($(M-\k-1.west)+(-0.2,0)$) {\i};
}

\node[on] at (M-1-1) {};
\node[on] at (M-2-1) {}; \node[on] at (M-2-2) {};
\foreach \c in {1,2,3} { \node[on] at (M-3-\c) {}; }
\foreach \c in {1,2,3,4} { \node[on] at (M-4-\c) {}; }
\node[on] at (M-5-5) {};
\foreach \c in {1,2,3,4,6} { \node[on] at (M-6-\c) {}; }
\foreach \c in {1,3,4,5,6,7} { \node[on] at (M-7-\c) {}; }
\node[on] at (M-8-8) {};
\node[on] at (M-9-9) {};
\foreach \c in {1,3,4,5,6,7,10} { \node[on] at (M-10-\c) {}; }

\foreach \k in {1,...,10} {
  \pgfmathtruncatemacro{\pid}{\k-1}
  \node[font=\Large, text=red!70!black] at ($(M-10-\k.south)+(0,-0.45)$) {$\pid$};
}
\node[font=\Large, text=red!70!black, anchor=east] at ($(M-10-1.south)+(-0.35,-0.52)$) {\texttt{pos}:};

\node[font=\Large, anchor=north, align=center] (step1title) at ($(M.south)+(0,-1.2)$)
  {\textbf{Step 1}: 2D Attention Mask};

\draw[->, line width=2pt, teal!50!black]
  ($(M.east)+(0.3,0)$) -- ++(4.0,0)
  node[midway, above=6pt, font=\large, text=black, align=center] {Remove masked tokens}
  node[midway, below=6pt, font=\large, text=black, align=center] {per round \& bundle};

\coordinate (RB) at ($(M.east)+(6.8, 0)$);

\node[font=\large\bfseries, anchor=east] at ($(RB)+(-0.4, 3.2)$) {Round 1:};

\node[qtok] (r1c1) at ($(RB)+(0.6, 3.2)$) {Ques};
\node[ftok, right=-\pgflinewidth of r1c1] (r1c2) {$F_1$};
\node[ttok, right=-\pgflinewidth of r1c2] (r1c3) {$T_1$};
\node[ttok, right=-\pgflinewidth of r1c3] (r1c4) {$T_2$};

\foreach \n/\p in {r1c1/0, r1c2/1, r1c3/2, r1c4/3} {
  \node[font=\Large, text=red!70!black] at ($(\n.south)+(0,-0.35)$) {$\p$};
}
\node[font=\Large, text=red!70!black, anchor=east] at ($(r1c1.south)+(-0.2,-0.42)$) {\texttt{pos}:};

\node[font=\large\bfseries, anchor=east] at ($(RB)+(-0.4, 0.7)$) {Round 2:};

\node[qtok] (r2c1) at ($(RB)+(0.6, 0.7)$) {Ques};
\node[ttok, right=-\pgflinewidth of r2c1] (r2c2) {$T_1$};
\node[ttok, right=-\pgflinewidth of r2c2] (r2c3) {$T_2$};
\node[ftok, right=-\pgflinewidth of r2c3] (r2c4) {$F_1$};
\node[ftok, right=-\pgflinewidth of r2c4] (r2c5) {$F_2$};
\node[ttok, right=-\pgflinewidth of r2c5] (r2c6) {$T_3$};

\foreach \n/\p in {r2c1/0, r2c2/1, r2c3/2, r2c4/3, r2c5/4, r2c6/5} {
  \node[font=\Large, text=red!70!black] at ($(\n.south)+(0,-0.35)$) {$\p$};
}
\node[font=\Large, text=red!70!black, anchor=east] at ($(r2c1.south)+(-0.2,-0.42)$) {\texttt{pos}:};

\node[font=\large\bfseries, anchor=east] at ($(RB)+(-0.4, -1.8)$) {Round 3:};

\node[qtok] (r3c1) at ($(RB)+(0.6, -1.8)$) {Ques};
\node[ttok, right=-\pgflinewidth of r3c1] (r3c2) {$T_1$};
\node[ttok, right=-\pgflinewidth of r3c2] (r3c3) {$T_2$};
\node[ftok, right=-\pgflinewidth of r3c3] (r3c4) {$F_1$};
\node[ftok, right=-\pgflinewidth of r3c4] (r3c5) {$F_2$};
\node[ttok, right=-\pgflinewidth of r3c5] (r3c6) {$T_3$};
\node[ftok, right=-\pgflinewidth of r3c6] (r3c7) {$F_3$};

\foreach \n/\p in {r3c1/0, r3c2/1, r3c3/2, r3c4/3, r3c5/4, r3c6/5, r3c7/6} {
  \node[font=\Large, text=red!70!black] at ($(\n.south)+(0,-0.35)$) {$\p$};
}
\node[font=\Large, text=red!70!black, anchor=east] at ($(r3c1.south)+(-0.2,-0.42)$) {\texttt{pos}:};

\coordinate (brace-x) at ([xshift=8mm]r3c7.east);
\draw[decorate, decoration={brace, amplitude=10pt}, line width=1pt]
  (brace-x |- r1c4.north) ++(0,4mm) -- (brace-x |- r3c7.south) ++(0,-6mm)
  node[midway, right=14pt, align=center, font=\large]
    {Bundled\\Training\\(shared $R$)};

\node[font=\large, text=gray!70!black, align=center, anchor=north]
  at ($(r3c4.south)+(0,-0.8)$)
  {Standard causal attention, contiguous \texttt{position\_ids}};

\node[font=\Large, anchor=north, align=center] at (step1title.north -| r3c4)
  {\textbf{Step 2}: Per-Round Traces};

\coordinate (leg) at ($(r3c2.south)+(-1.5,-1.7)$);
\node[font=\Large\bfseries, anchor=south west] at ($(leg)+(-1.0,0.75)$) {};
\node[qtok, minimum width=6mm, minimum height=4mm] (lq) at (leg) {};
\node[font=\Large, right=2pt of lq] (lqt) {Question};
\node[ftok, minimum width=6mm, minimum height=4mm, right=12pt of lqt] (lf) {};
\node[font=\Large, right=2pt of lf] (lft) {Insight};
\node[ttok, minimum width=6mm, minimum height=4mm, right=12pt of lft] (lt) {};
\node[font=\Large, right=2pt of lt] {Thinking};

\end{tikzpicture}%
}
\caption{Two-step training construction for the alternative multi-round method. Step~1 applies a round-specific visibility mask over the concatenated trajectory; Step~2 removes masked tokens round by round and bundles the resulting traces so that training-time attention and position indices match rollout.}
\label{fig:2d-mask-training}
\end{figure*}
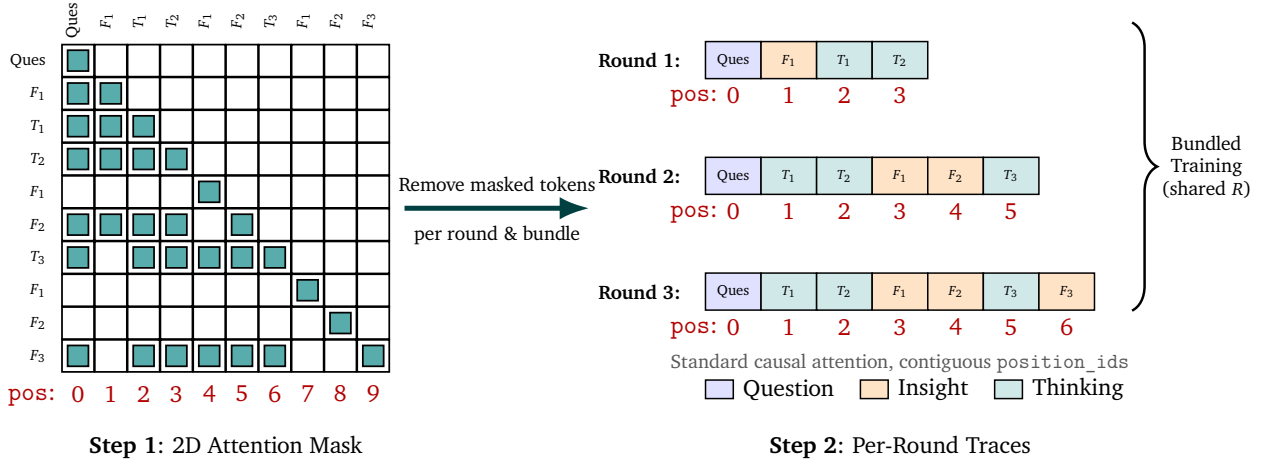

\subsection{SFT Data Construction Details}
\paragraph{Source Data.}
We sample $2{,}000$ math problems from DAPO-Math-17K~\citep{dapo}, a
large-scale mathematical reasoning dataset.

\paragraph{Solution Generation.}
For each sampled problem, we use \texttt{Qwen3.5-397B-A17B} to generate
a complete solution in non-think mode, with temperature $\tau=0.7$ and
top-$p=0.95$.

\paragraph{Post-hoc insight Construction.}
Each generated solution is then decomposed into a multi-round insight
format as follows:
\begin{enumerate}
    \item The full solution is split into logical reasoning steps (up to $8$ steps).
    \item Each reasoning step is summarized into a concise insight (1--2 sentences, up to $256$ tokens) via a separate model call.
    \item Multi-turn SFT entries are constructed so that each round contains a reasoning step in \texttt{<think>}\allowbreak\texttt{...}\allowbreak\texttt{</think>} followed by an intermediate conclusion in \texttt{<finding>}\allowbreak\texttt{...}\allowbreak\texttt{</finding>} tags, with the final round producing the answer.
\end{enumerate}
This process yields $1{,}892$ valid cases and $8{,}102$ SFT training
entries. The average number of insight rounds per problem is $3.3$,
with a maximum of $9$ rounds.

\subsection{insight Format: Detailed Example}
\label{appendix:insight_example}

This subsection illustrates the multi-round insight format with a
concrete example from our RL training rollout. The model solves the
problem in $3$ rounds: two insight rounds followed by one answer round.
At each round, the model receives all accumulated context (previous
reasoning and insights) and generates the next step.

\paragraph{Problem.}

\begin{mdframed}[backgroundcolor=gray!10]
Define a sequence $\{a_n\}_{n\geq 1}$ as follows:
\begin{itemize}
    \item For $n \leq 2$: $a_n = n - 1$
    \item For all $n \geq 2$: $a_n$ is the remainder of $a_{n-1} + a_{n-2}$ divided by $3$
\end{itemize}
Compute $\displaystyle\sum_{i=2018}^{2025} a_i$.
\end{mdframed}

\subsubsection*{Round 1: Insight Round}

\paragraph{Input.} The model receives the original problem as the user prompt.

\paragraph{Model Output.}

\begin{mdframed}[backgroundcolor=blue!5]
\texttt{<think>}\\
First, compute the initial terms of the sequence:
\begin{align*}
a_1 &= 0,\quad a_2 = 1,\quad a_3 = (1+0) \bmod 3 = 1, \\
a_4 &= (1+1) \bmod 3 = 2,\quad a_5 = (2+1) \bmod 3 = 0, \\
a_6 &= (0+2) \bmod 3 = 2,\quad a_7 = (2+0) \bmod 3 = 2, \\
a_8 &= (2+2) \bmod 3 = 1,\quad a_9 = (1+2) \bmod 3 = 0, \\
a_{10} &= (0+1) \bmod 3 = 1.
\end{align*}
\texttt{</think>}

\texttt{<finding>}The sequence $\{a_n\}$ exhibits periodic behavior starting from $n=3$, with values $1, 2, 0, 2, 2, 1$ repeating cyclically. The period is 6.\texttt{</finding>}
\end{mdframed}

The model computes the first $10$ terms and identifies periodicity.
Note that the insight contains an \emph{incorrect} period estimate ($6$
instead of $8$)---the model will self-correct in the next round.

\subsubsection*{Round 2: Insight Round}

\paragraph{Input.} The model receives the original problem concatenated with the full Round~1 reasoning (\texttt{<think>} content) and the Round~1 insight as accumulated context.

\paragraph{Model Output.}

\begin{mdframed}[backgroundcolor=blue!5]
\texttt{<think>}\\
Examining the first $10$ terms: $a_1=0, a_2=1, a_3=1, a_4=2, a_5=0, a_6=2, a_7=2, a_8=1, a_9=0, a_{10}=1$.

Comparing $(a_1, a_2)$ and $(a_9, a_{10})$: both are $(0, 1)$. Since this is a linear recurrence determined by consecutive pairs, the sequence repeats from this point. The period length is $T = 9 - 1 = 8$, and the terms within one period are: $0, 1, 1, 2, 0, 2, 2, 1$.

The sum over one period: $S_{\text{cycle}} = 0 + 1 + 1 + 2 + 0 + 2 + 2 + 1 = 9$.

We need $\sum_{i=2018}^{2025} a_i$. First, determine $a_{2018}$'s position in the period: $2018 \bmod 8$. The mapping $n \bmod 8$: $1 \to a_1(0),\ 2 \to a_2(1),\ \ldots,\ 0 \to a_8(1)$.
\texttt{</think>}

\texttt{<finding>}The sequence is periodic with period 8, repeating $0, 1, 1, 2, 0, 2, 2, 1$. The sum of one period is 9.\texttt{</finding>}
\end{mdframed}

Building on the first round's computation, the model re-examines the
sequence, corrects the period to $8$, computes the period sum, and sets
up the index mapping for the final calculation.

\subsubsection*{Round 3: Answer Round}

\paragraph{Input.} The model receives the original problem, the full reasoning from Rounds~1 and~2, and all accumulated insights.

\paragraph{Model Output.}

\begin{mdframed}[backgroundcolor=green!5]
\texttt{<think>}\\
Compute $2018 \pmod{8}$: $2018 = 8 \times 252 + 2$, remainder is 2.

So $a_{2018}$ corresponds to the 2nd term in the period.

The terms $a_{2018}, \ldots, a_{2025}$ map to positions $\{2,3,4,5,6,7,8,1\}$, which is exactly a permutation of $\{1,2,3,4,5,6,7,8\}$.

Therefore, their sum equals $S_{\text{cycle}} = 9$.

Answer: 9
\texttt{</think>}

The sequence $\{a_n\}$ is periodic with period $T = 8$ and per-period sum $S = 9$. Since $2018 \equiv 2 \pmod{8}$, the 8 terms $a_{2018}$ through $a_{2025}$ span exactly one complete period, so their sum equals $9$.

Answer: $\boxed{9}$
\end{mdframed}

\subsection{SFT Hyperparameters}
We fine-tune for one epoch with the following setup:
\begin{itemize}
    \item Learning rate: $1 \times 10^{-5}$
    \item Maximum sequence length: $40{,}000$ tokens
    \item Global batch size: $128$
    \item Micro batch size per GPU: $4$
    \item Hardware: $8$ GPUs with FSDP
    \item Gradient checkpointing enabled
    \item No padding mode with right truncation
\end{itemize}

\subsection{RL Setup}
We apply GRPO~\citep{grpo} in two settings: a insight-enabled rollout
and a insight-disabled baseline. The RL training set is DAPO-Math-15K,
a $15{,}917$-problem subset of DAPO-Math-17K. Evaluation is conducted
on AIME 2025~\citep{maa_aime}, with $8$ rollouts per problem and mean@1
accuracy as the validation metric.

\paragraph{Reward Function.}
We use the DAPO reward manager with the same soft overlong penalty as in
the main post-training experiment. The model's answer is extracted from
the last \texttt{Answer:} line, normalized, and compared against the
ground truth, yielding a correctness reward of $+1.0$ for correct
answers and $-1.0$ for incorrect ones. In addition, we enable the soft
overlong penalty to discourage trajectories that approach the maximum
response budget: \texttt{overlong\_buffer.enable=True}, the penalty
factor is $1.0$, and \texttt{max\_resp\_len} passed to the reward
matches the rollout \texttt{max\_response\_length}.

\paragraph{RL Hyperparameters.}
Both experiments use the following shared configuration on $8$
NVIDIA H200 GPUs:
\begin{itemize}
    \item Algorithm: GRPO with token-level loss aggregation
    \item Sequence parallelism (Ulysses SP): $8$
    \item Rollout tensor parallelism: $2$
    \item Global training batch size: $64$ prompts
    \item Rollout $n$: $8$ responses per prompt ($512$ total samples per step)
    \item PPO mini-batch size: $32$
    \item Learning rate: $1 \times 10^{-6}$ with $10$ warmup steps
    \item KL loss coefficient: $0.001$
    \item Maximum prompt length: $27{,}000$ tokens
    \item Maximum response length: $25{,}000$ tokens
    \item Total training: $5$ epoch
    \item Validation frequency: every $20$ steps
\end{itemize}

\paragraph{Insight-Mode Configuration.}
In the insight-enabled experiment, rollouts use multi-round generation
with up to $5$ insight rounds per problem. At each round, the model
produces a reasoning step and an intermediate conclusion; all previous
reasoning steps and insights remain in context for the next round. The
rollout continues until either the maximum number of turns is reached or
the model fails to emit a well-formed
\texttt{<finding>}\allowbreak\texttt{...}\allowbreak\texttt{</finding>}
tag.

\paragraph{Baseline Configuration.}
In the insight-disabled experiment, the model generates a single
continuous response per problem without any multi-round structure.

\subsection{Controlled Variables}
To ensure a fair comparison, both experiments:
\begin{itemize}
    \item use identical RL hyperparameters;
    \item are evaluated on the same AIME 2025 benchmark with the same $8$-rollout mean@1 protocol; and
    \item differ only in whether the multi-round insight rollout is enabled.
\end{itemize}

\subsection{Experiment Results}
Figure~\ref{fig:alternative-training-results} summarizes the validation
performance of the alternative training method over a unified training
budget of $420$ steps. As in the main post-training comparison, we track
\(\mathrm{best@32}\), \(\mathrm{maj@32}\), and \(\mathrm{mean@32}\)
on AIME~2025 throughout training.

The outcome is mixed. By the end of the $420$-step budget, the
InsightReplay variant overtakes the baseline on both
\(\mathrm{mean@32}\) and \(\mathrm{best@32}\), indicating a late
advantage in average sample quality and upper-tail performance.
However, \(\mathrm{maj@32}\) remains consistently below the baseline,
suggesting that this training recipe does not improve sample-level
consistency in the same way.

We also emphasize that these experiments were stopped after exactly
$420$ steps rather than trained further. In addition, this direct
multi-round training mode is substantially more expensive
operationally. Each InsightReplay rollout must merge multiple
rounds into a single training example, and the prompt grows longer at
later rounds because more reasoning and insight tokens are carried
forward. As a result, token cost increases over the course of a single
rollout. For this reason, despite the late crossover on
\(\mathrm{mean@32}\) and \(\mathrm{best@32}\), we ultimately did not
adopt this training strategy in the main paper.

\begin{figure*}[t]
\centering
\includegraphics[width=\textwidth]{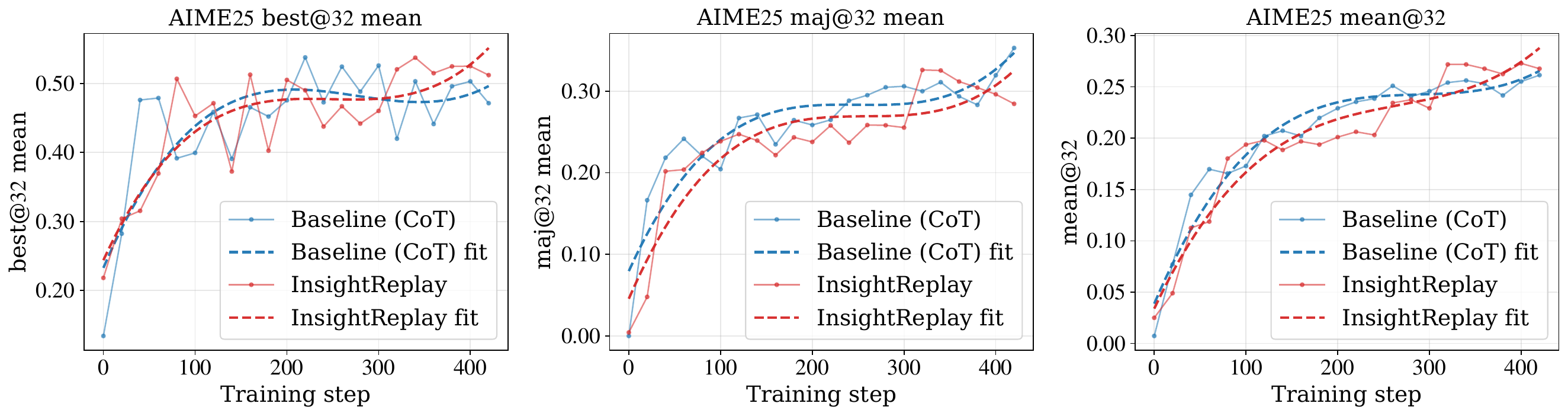}
\caption{Validation accuracy on AIME~2025 for the alternative training
method over a unified $420$-step budget. We report
\(\mathrm{best@32}\), \(\mathrm{maj@32}\), and \(\mathrm{mean@32}\)
throughout training.}
\label{fig:alternative-training-results}
\end{figure*}

\end{document}